\documentclass[11pt]{article}
\usepackage[margin=1in]{geometry}
\usepackage[usenames,dvipsnames]{xcolor}

\definecolor{Gred}{RGB}{219, 50, 54}
\definecolor{Ggreen}{RGB}{60, 186, 84}
\definecolor{Gblue}{RGB}{72, 133, 237}
\definecolor{Gyellow}{RGB}{247, 178, 16}
\definecolor{ToCgreen}{RGB}{0, 128, 0}
\definecolor{myGold}{RGB}{231,141,20}
\definecolor{myBlue}{rgb}{0.19,0.41,.65}
\definecolor{myPurple}{RGB}{175,0,124}

\usepackage{hyperref}
\hypersetup{
			colorlinks=true,
	citecolor=ToCgreen,
	linkcolor=Sepia,
	filecolor=Gred,
	urlcolor=Gred
	}

\usepackage{preamble}
\title{Learning Deep ReLU Networks Is Fixed-Parameter Tractable}
\author{Sitan Chen\thanks{This work was supported in part by a Paul and Daisy Soros Fellowship, NSF CAREER Award CCF-1453261, and NSF Large CCF-1565235.} \\
\texttt{sitanc@mit.edu}\\
MIT
\and Adam R. Klivans\thanks{Supported by NSF awards AF-1909204, AF-1717896, and the NSF AI Institute for Foundations of Machine Learning (IFML). Work done while visiting the Institute for Advanced Study, Princeton, NJ.} \\
\texttt{klivans@cs.utexas.edu} \\
UT-Austin and IAS
\and Raghu Meka\thanks{Supported by NSF CAREER Award CCF-1553605.} \\
\texttt{raghum@cs.ucla.edu} \\
UCLA
}
\newcommand{\N}{\mathcal{N}}

\newcommand{\sgn}{\mathop{\textrm{sgn}}}
\newcommand{\vM}{\vec{M}}
\newcommand{\vW}{\vec{W}}

\newcommand{\td}{\widetilde}

\newcommand{\ulam}{\underline{\lambda}}
\newcommand{\calC}{\mathcal{C}}
\newcommand{\calJ}{\mathcal{J}}
\newcommand{\op}{\mathop{\textrm{op}}}
\newcommand{\Sig}{\vec{\Sigma}}
\newcommand{\gap}{\mathsf{gap}}
\newcommand{\emp}{\mathsf{emp}}

\newcommand{\condE}{\mathbb{E}\expectarg}
\DeclarePairedDelimiterX{\expectarg}[1]{[}{]}{%
  \ifnum\currentgrouptype=16 \else\begingroup\fi
  \activatebar#1
  \ifnum\currentgrouptype=16 \else\endgroup\fi
}

\newcommand{\clip}{\mathsf{clip}}

\newcommand{\innermid}{\nonscript\;\delimsize\vert\nonscript\;}
\newcommand{\activatebar}{%
  \begingroup\lccode`\~=`\|
  \lowercase{\endgroup\let~}\innermid 
  \mathcode`|=\string"8000
}

\SetKw{Break}{break}

\newcommand{\normbound}{B}
\newcommand{\Lip}{\Lambda}
\newcommand{\Cjunta}{C_{\mathsf{piecewise}}}
\newcommand{\Cnet}{C_{\mathsf{network}}}
\newcommand{\Constant}{C_*}

\renewcommand{\epsilon}{\varepsilon}

\begin{document}

\maketitle

\begin{abstract}
    We consider the problem of learning an unknown ReLU network with respect to Gaussian inputs and obtain the first nontrivial results for networks of depth more than two.  We give an algorithm whose running time is a fixed polynomial in the ambient dimension and some (exponentially large) function of only the network's parameters.  
    
    Our bounds depend on the number of hidden units, depth, spectral norm of the weight matrices, and Lipschitz constant of the overall network (we show that some dependence on the Lipschitz constant is necessary).  We also give a bound that is doubly exponential in the size of the network but is independent of spectral norm.   These results provably cannot be obtained using gradient-based methods and give the first example of a class of efficiently learnable neural networks that gradient descent will fail to learn.

    In contrast, prior work for learning networks of depth three or higher requires \emph{exponential} time in the ambient dimension, even when the above parameters are bounded by a constant.  Additionally, all prior work for the depth-two case requires well-conditioned weights and/or positive coefficients to obtain efficient run-times.  Our algorithm does not require these assumptions.
    
    Our main technical tool is a type of filtered PCA that can be used to iteratively recover an approximate basis for the subspace spanned by the hidden units in the first layer.  Our analysis leverages new structural results on lattice polynomials from tropical geometry. 
\end{abstract}


\section{Introduction}

We study the problem of learning the following class of concepts:

\begin{definition}[ReLU Networks]\label{defn:relunets}
	Let $\calC_S$ denote the concept class of (feedforward) ReLU networks over $\R^d$ of size $S$. Specifically, $F\in\calC_S$ if there exist weight matrices $\vW_0\in\R^{k_0\times d}, \vW_1\in\R^{k_1\times k_0},\ldots, \vW_{L}\in\R^{k_{L}\times k_{L-1}}, \vW_{L+1}\in\R^{1\times k_{L}}$ for which \begin{equation}
		F(x) \triangleq \vW_{L+1}\phi\left(\vW_{L}\phi\left(\cdots \phi(\vW_0 x) \cdots \right)\right),
	\end{equation} where $\phi(z)\triangleq \max(z,0)$ is the ReLU activation applied entrywise, and $k_0+\cdots + k_{L} = S$. In this case we say that $F$ is computed by a ReLU network with depth $L+2$. We will refer to the rank of $\vW_0$ as $k$, to emphasize that the value of $F$ only depends on a $k$-dimensional subspace of $\R^d$. We will also let $k_{L+1} = 1$.
\end{definition}

	When the weight matrices of two ReLU networks $F,F'\in\calC_S$ have the same dimensions (at all layers), then we say that $F$ and $F'$ \emph{have the same architecture}.

For example, a depth two ReLU network of size $S$ in $d$-dimensions is a function  $F:\R^d \rightarrow \R$ of the form \begin{equation}
	F(x) = \sum_{i=1}^S \lambda_i \phi(\iprod{w_i,x}),
\end{equation}
where $\lambda_i \in \R$ are scalars and $w_i \in \R^d$ are arbitrary vectors. 

Note that any Boolean function $F:\brc{\pm 1}^n\to\brc{\pm 1}$ can be computed by an $n$-layer ReLU network (see Lemma~\ref{lem:representation} in Appendix~\ref{app:represent}). In particular, if $F$ is a junta depending only on $k$ variables, then it can be computed by a $k$-layer ReLU network with size that depends only on $k$.

\paragraph{Learning ReLU Networks}
The problem of PAC learning an unknown ReLU network from labeled examples is a central challenge in the theory of machine learning. Given samples from a distribution of the form $(x,y) \in \R^d \times \R$ where $y = F(x)$ with $F$ an unknown size-$S$ ReLU network, and $x$ is drawn according to a distribution $\cal{D}$, the goal is to output a function $f:\R^d \rightarrow \R$ with small {\em test} error, i.e., $\E[x,y]{(y - f(x))^2} \leq \epsilon \E{y^2}$.  In this work, we focus on the widely studied case where the input distribution on $x$ is Gaussian.  

Ideally, we would like an algorithm with sample complexity and running time that is polynomial in all the relevant parameters.  As a first step, the algorithm should depend polynomially on the {\em dimension} (it is often easy to obtain brute-force search algorithms that run in time exponential in the dimension\footnote{Although in our specific case even this type of search turns out to be nontrivial.}). Even this goal, however, has been elusive: it is not known how to achieve subexponential-time algorithms for general depth two ReLU networks (without making additional assumptions on the network). 

In this work, we give the first algorithm for learning ReLU networks whose running time is a fixed polynomial in the dimension, regardless of the depth of the network.  Our algorithm is \emph{fixed-parameter tractable}: we show that we can \emph{properly learn} (i.e., the output hypothesis is also a ReLU network) ReLU networks with sample complexity and running time that is a fixed polynomial in the dimension and an exponential function of the network's {\em parameters}.


More precisely, our main result is as follows. We will also make the (as it turns out necessary) assumption that the ReLU network has a bounded Lipschitz constant: a function $f:\R^d \rightarrow \R$ is $\Lip$-Lipschitz if $|f(x) - f(x')| \leq \Lip \|x - x'\|_2$ for all $x,x'$. 

\begin{theorem}[Main, see Theorem~\ref{thm:main_nets} for formal statement]\label{thm:main_informal}
Let $\calD$ be the distribution over pairs $(x,y)\in \R^d \times \R$ where $x \sim \calN(0,\Id)$ and $y = F(x)$ for a size-$S$ ReLU network $F$ with depth $L+2$, Lipschitz constant at most $\Lip$, rank of \emph{bottom} weight matrix $\vW_0$ being $k$, and whose weight matrices all have spectral norm at most $\normbound$.

There is an algorithm that 
draws $d\log(1/\delta) \exp\left(\poly(k,S,\Lip/\epsilon)\right) \normbound^{O(Lk)}$ samples, runs in time $\td{O}(d^2\log(1/\delta))\exp\left(\poly(k,S,\Lip/\epsilon)\right) \normbound^{O(L kS^2)}$, and outputs a ReLU network $\td{F}$ such that $\E{(y - \td{F}(x))^2} \leq \epsilon$ with probability at least $1-\delta$.\footnote{See Remark~\ref{remark:scaleinvariant} for a discussion of why this guarantee is scale-invariant.}
\end{theorem}

Note that the sample complexity is linear while the run-time is quadratic in the ambient dimension. In particular, in the well-studied special case where the product of the spectral norms of the weight matrices is a constant (see e.g. \cite{golowich2018size}), in which case the Lipschitz constant of the network is also constant, we can obtain the following corollary:

\begin{corollary}
Let $\calD$ be the distribution over pairs $(x,y)\in \R^d \times \R$ where $x \sim \calN(0,\Id)$ and $y = F(x)$ for a size-$S$ ReLU network $F$ for which the product of the spectral norms of its weight matrices is a constant.

Then there is an algorithm that draws $N = d\log(1/\delta)\exp(O(k^3/\epsilon^2 + kS))$ samples, runs in time $\td{O}(d^2\log(1/\delta)) \exp(O(k^3S^2/\epsilon^2 + kS^3))$, and outputs a ReLU network $\td{F}$ such that $\E{(y - \td{F}(x))^2} \leq \epsilon$ with probability at least $1-\delta$.
\end{corollary}

As mentioned earlier, no algorithms that were sub-exponential in $d$ were known even for $S, \normbound,\epsilon$ being constants.

Before going further, we note that a dependence on the Lipschitz constant of the network is necessary even for learning depth two ReLU networks with respect to Gaussians:

\begin{example}\label{example:spike}
Let $\Lip>0$. Consider the size-3, depth two ReLU network $F:\R^2\to\R$ given by \begin{equation}
	F(x_1,x_2) = \phi(x_1 + \Lip x_2) + \phi(3x_1 + \Lip x_2) - 2\phi(-x_1 + \Lip x_2).
\end{equation}
The Lipschitz constant of $F$ is $\Theta(\Lip)$: $F(0,1/\Lip) = 1$ and $F(1,1/\Lip) = 2$. Furthermore, note that for $(x_1,x_2)\in\S^1$, $F(x_1,x_2) = 0$ unless $x_2\in [-3/\Lip,3/\Lip]$. By rotational symmetry, for $(x_1,x_2)\sim\N(0,\Id)$, $F(x_1,x_2) \neq 0$ with probability at most $O(1/\Lip)$.
\end{example}

Note that for depth two ReLU networks with positive weights, no such dependence on the Lipschitz constant is necessary intuitively because without cancellations between the hidden units, one cannot devise ``spiky'' functions $F$ which simultaneously have small variance but attain a large value at some bounded-norm $x$.

Interestingly, our techniques are also general enough to handle general continuous piecewise-linear functions (see Definition~\ref{def:piecewise} for a formal definition):

\begin{theorem}[See Theorem~\ref{thm:main_piecewise} for formal statement]\label{thm:piecewise_informal}
Let $\calD$ be the distribution over pairs $(x,y)\in \R^d \times \R$ where $x \sim \calN(0,\Id)$ and $y = F(x)$ for a continuous piecewise-linear function $F$ which only depends on the projection of $x$ to a $k$-dimensional subspace $V$, has at most $M$ linear pieces, and is $\Lip$-Lipschitz.

There is an algorithm that draws $d\log(1/\delta)\cdot \poly\left(\exp\left(k^3\Lip^2/\epsilon^2\right),M^k\right)$ samples, runs in time $\td{O}(d^2\log(1/\delta))\cdot M^{M^2}\cdot \poly\left(\exp\left(k^4\Lip^2/\epsilon^2\right),M^{k^2}\right)$, and outputs a piecewise-linear function $\td{F}$ such that $\E{(y - \td{F}(x))^2} \leq \epsilon$ with probability at least $1-\delta$.
\end{theorem}

Note that a size-$S$ ReLU network is a continuous piecewise-linear function with at most $2^S$ linear pieces. Specializing Theorem~\ref{thm:piecewise_informal} to ReLU networks gives a guarantee which is incomparable to Theorem~\ref{thm:main_informal}: we obtain an algorithm that depends doubly exponentially on $S$ but has no dependence on the norms of the weight matrices.

\subsection{Prior Work on Provably Learning Neural Networks}


\paragraph{Algorithmic Results}
Algorithms for learning neural networks (obtaining small {\em test error}) have been intensely studied in the literature.  In the last few years alone there have been many papers giving provable results for learning restricted classes of neural networks under various settings  \cite{janzamin2015beating,zhang2016l1,zhong2017recovery,brutzkus2017globally,goel2017reliably,LiY17,zhangps17,Tian17,convotron,Duconv,ge2018learning,ge2018learning2,manurangsi2018computational,bakshi2019learning,goel2019learning,azll,vempala2019gradient,zgu,diakonikolas2020approximation,sewoong,LiMZ20}. 

The predominant techniques are spectral or tensor-based dimension reduction \cite{janzamin2015beating,zhong2017recovery,bakshi2019learning,diakonikolas2020algorithms}, kernel methods \cite{zhang2016l1,goel2017reliably,Daniely17,manurangsi2018computational,goel2019learning}, and gradient-based methods \cite{ge2018learning,ge2018learning2,vempala2019gradient}.  All prior work takes distributional and/or architectural assumptions, the most common one being that the inputs come from a standard Gaussian. We will also work in this setting.\footnote{Other works such as \cite{azll} or kernel-based methods \cite{zhang2016l1,goel2017reliably} require strong norm-based assumptions on the inputs and weights.}




As pointed out in \cite{goel2020superpolynomial,diakonikolas2020approximation}, all existing algorithmic results for Gaussian inputs hold {\em only for depth two networks} and make at least one of two assumptions on the unknown network $F$ in question: \begin{itemize}
	\setlength{\itemindent}{1in}
	\item[\textbf{Assumption (1)}] Weight matrix $\vW_0$ is well-conditioned and, in particular, full rank.
	\item[\textbf{Assumption (2)}] The vector at the output layer ($\vW_1$ when $L = 0$) has all positive entries.
\end{itemize}

Assumption (1) allows one to use tensor decomposition to recover the parameters of the network and hence PAC learn, an idea that has inspired a long line of works \cite{janzamin2015beating,zhong2017recovery,ge2018learning,ge2018learning2,bakshi2019learning}. However, the assumption is not necessary for PAC learning or achieving low-prediction error. For instance, consider a pathological case where $\vW_0$ has repeated rows. Here, while parameter recovery is not possible it is still possible to PAC learn. To our knowledge, the only work that can PAC learn depth two networks over Gaussian inputs without a condition number bound on $\vW_0$ is \cite{diakonikolas2020algorithms}. However, their work still requires assumption (2) (and only holds for depth two networks).  Our work shows that assumption (2) is neither information-theoretically nor computationally necessary.



\paragraph{Limitations of Gradient-Based Methods}

Two recent works \cite{goel2020superpolynomial,diakonikolas2020algorithms} showed that a broad family of algorithms, namely \emph{correlational statistical query (CSQ) algorithms}, fail to PAC learn even depth two ReLU networks; that is, functions of the form $F(x) = \sum^k_{i=1}\lambda_i \phi(\iprod{v_i,x})$ with respect to Gaussian inputs in time polynomial in $d$ where $d$ is the ambient dimension (in fact, \cite{diakonikolas2020algorithms} rules out running time $d^{o(k)}$).  Informally, a CSQ algorithm is limited to using noisy estimates of statistics of the form $\E{y\cdot \sigma(x)}$ for arbitrary bounded $\sigma$, where the expectation is over examples $(x,y)$ and $y = F(x)$ is computed by the network. The point is that this already rules out a wide range of algorithmic approaches in theory and practice, including gradient descent on overparameterized networks (i.e., using neural tangent kernels \cite{jacot} or the  mean-field approximation for gradient dynamics \cite{meimontanaringuyen}). Note that the algorithms of \cite{diakonikolas2020algorithms} for learning depth two ReLU networks with positive coefficients are CSQ algorithms as well.

Note that as a consequence of Theorem~\ref{thm:main_informal}, for any $\epsilon$ a function of $k$, our algorithm can learn the lower bound instances in \cite{goel2020superpolynomial,diakonikolas2020algorithms} to error $\epsilon$ in time $g(k)\cdot\poly(d)$ for some $g$ (note that the norm bounds and Lipschitz constants for these instances are upper bounded by functions of $k$), which is impossible for any CSQ algorithm. We explain why our algorithm is not a CSQ algorithm in Section \ref{sec:notsq}.

For the classification version of this problem (i.e., taking a softmax) where we observe $Y \in \{0,1\}$ such that $\E{Y|X} = \sigma(f(X))$ where $\sigma$ is say sigmoid and $f(X)$ is a depth two ReLU network, Goel et al. \cite{goel2020superpolynomial} show that even general SQ algorithms cannot achieve a runtime with polynomial dependence on the dimension.  We also remark there is an extensive literature of previous work showing various hardness results for learning certain classes of neural networks \cite{blum1989training,vu2006infeasibility,klivans2009cryptographic,livni2014computational,goel2017reliably,song2017complexity,shalev2017failures,shamir2018distribution,vempala2019gradient,goel2019time,daniely2020hardness}. We refer the reader to \cite{goel2020superpolynomial} for a discussion of how these prior works relate to the above CSQ lower bounds.

\subsection{Other Related Work}

\paragraph{Multi-Index Models} Functions computed by ReLU networks where $\vW_0$ has fewer rows than columns are a special case of a \emph{multi-index model}, that is, a function $F:\R^d\to \R$ given by $F(x) = f(\vW^{\top} x)$ for some matrix $\vW\in\R^{k\times d}$ and some function $f:\R^k\to \R$. In the theoretical computer science literature, these are sometimes referred to as \emph{subspace juntas} \cite{vempala2011structure,de2019your}.

One of the strongest results in this line of work, and the closest in spirit to the setting we consider, is that of \cite{dudeja2018learning}, which gives various conditions on $f$  under which one can recover $\vW$ (under Gaussian inputs) in the special case where $k = 1$, as well as a vector in the row span of $\vW$ in the case of general $k$ (although these results do not hold for ReLU). In general, the literature on multi-index models is vast, and we refer to \cite{dudeja2018learning} for a comprehensive overview of this body of work. Many works were inspired by a simple but powerful connection to Stein's lemma \cite{li1992principal,brillinger2012generalized,plan2016generalized}, which was also a key ingredient in the above algorithms for learning neural networks using tensor decomposition. One technique in this literature which is somewhat similar in spirit to the techniques we employ in this work is that of \emph{sliced inverse regression} \cite{babichev2018slice,li1991sliced}, and we elaborate in Remark~\ref{remark:slice} on this connection.

\paragraph{Piecewise-Linear Regression} Lastly, we mention that previous works on \emph{segmented regression} (see e.g. \cite{acharya2016fast} on the references therein) study regression for piecewise-linear functions but work with a different notion of piecewise-linearity that is unrelated to our setting.


\section{Proof Overview}
\label{sec:overview}

Suppose we are given samples $(x,y)$ where $y = F(x)$ is computed by a size $S$ ReLU network as in Definition \ref{defn:relunets}. Let $V\subseteq \R^d$ denote the span of the rows of $\vW_0$ and let $k$ be its dimension. We will call $V$ the \emph{relevant subspace}, because the value of $F$ only depends on the projection of $x$ to $V$. In particular, we can write $y = F'(\Pi_V(x))$ for some function $F': V \to \R$ that is itself a size $S$ ReLU network and $\Pi_V$ denotes the projection operator onto $V$. The main focus of our algorithm will be in figuring out the relevant subspace $V$ given samples $(x,y)$. This is the hardest part of the algorithm, because once we learn the relevant subspace to high enough accuracy, we can grid-search over ReLU networks in this subspace. Even this grid search turns out to be non-trivial to analyze and entails proving new \emph{stability} results for piecewise-linear functions.

\paragraph{Filtered PCA} \label{sec:notsq}

Our algorithm builds upon the \emph{filtered PCA} approach, originally introduced in \cite{chen2020learning} for the purposes of learning low-degree polynomials over Gaussian space.\footnote{For readers familiar with the approach there, we explain in Remark~\ref{remark:compare} why a straightforward application of the algorithm there cannot work.} For any $\psi:\R \to \R$, let $\vM_{\psi} \triangleq \E{\psi(Y) (X X^T - \Id)}$. A basic but important observation is that for any choice of $\psi$, all vectors orthogonal to the true subspace $V$ are in the kernel of $\vM_{\psi}$. A natural idea for identifying the true subspace then is to look at the nonzero singular vectors of $\vM_{\psi}$ for a suitable $\psi$. If we could show that $\vM_{\psi}$ has $k$ nonzero singular values all bounded away from $0$ by some dimension-independent margin $c(\psi)$, then we could hope to approximately recover $V$ by empirically estimating $\vM_{\psi}$ using $O(d/c(\psi)^2)$, invoking standard matrix concentration, and computing its top-$k$ singular subspace. So the main hurdle is to identify an appropriate $\psi$ for which this is the case.


What should the $\psi$ be? For instance if $\psi$ is the identity function, then the matrix $\vM_{\psi}$ could be identically zero. This is an essential difference between our setting and the setting studied in previous works \cite{diakonikolas2020algorithms,ge2018learning} (in the $L = 0$ case) where the output layer's coefficients are all positive, for which this choice of $\psi$ would suffice to recover the relevant subspace.

Note that this is consistent with the CSQ lower bounds of \cite{goel2020superpolynomial,diakonikolas2020algorithms}, as any algorithm that just tries to use the spectrum of $\vM_{\psi}$ for $\psi$ being the identity function would be a CSQ algorithm. Indeed, for any of the `hard' functions $F$ from those works which are ReLU networks with $L=0$ we would have $\vM_{\psi} = 0$ if $\psi$ is the identity function. 


We will choose $\psi$ not equal to the identity, and in this way our algorithm will be non-CSQ and evade the aforementioned CSQ lower bounds.

\paragraph{Threshold Filter.} Motivated by \cite{chen2020learning}, our starting point in the present work is to consider $\psi$ given by a univariate threshold, that is, $\psi(z) = \bone{\abs{z} > \tau}$ for suitable $\tau$. For brevity, for $\tau \in \R$ define $\vM_\tau = \E[x,y]{\bone{\abs{y} > \tau} (xx^T - \Id)}$. Then we have that \begin{equation}
	\iprod{\Pi_V, \vM_{\tau}} = \E[x,y]*{\bone{\abs{y} > \tau}\cdot(\norm{\Pi_V x}^2 - k)}.
\end{equation} In particular, if one could choose $\tau$ for which $\abs{F(x)} > \tau$ only if $\norm{\Pi_V x}^2 \ge 2k$ \footnote{The choice of $2k$ here is for exposition; any bound noticeably more than $k$, e.g., $k+1$ will do.}, then we would conclude that $\iprod{\Pi_V, \vM_{\tau}} \ge k\cdot\Pr{\abs{y} > \tau}$, so some singular value of $\vM_{\tau}$ is at least $\Pr{\abs{y} > \tau}$. If $F$ is $\Lip$-Lipschitz, we can simply choose $\tau$ to be $\sqrt{2k}\cdot \Lip$, and provided $\Pr{\abs{y} > \tau}$ is reasonably large, then we conclude that $\vM_{\tau}$ has some reasonably large singular value. Finally, to lower bound $\Pr{\abs{y} > \tau}$, we prove an \emph{anti-concentration} result for piecewise linear functions over Gaussian space (Lemma~\ref{lem:anti1}).

In other words, if one conditions on the samples $(x,y)$ whose responses $y$ are sufficiently large in magnitude, then we show that the resulting distribution is noticeably non-Gaussian in some direction, and by taking the top singular vector of the conditional covariance, we can approximately recover some direction inside the relevant subspace $V$.\footnote{Note that while the goal is to reweight the distribution over $x$ to look non-Gaussian in some relevant direction, the main challenge once we've fixed a reweighting is not to identify that non-Gaussian subspace, which in our setting is trivial and does not require any of the more sophisticated techniques in the non-Gaussian component analysis literature (e.g. \cite{vershynin2010introduction,goyal2019non}), but to argue that the new distribution is indeed non-Gaussian in some direction in $V$. In a similar vein, while the work \cite{vempala2011structure} gives some moment-based conditions under which it is possible to learn multi-index models over Gaussian inputs, it seems highly nontrivial to verify whether such conditions actually hold for ReLU networks, and in addition their results seem tailored to $\brc{0,1}$-valued functions.}

Unfortunately, all that the above analysis tells us is that the trace of $\vM_{\tau}$ is non-negligible which in turn helps us guarantee that we identify at least one direction in $V$. It is not at all clear whether the above threshold approach is enough to identify more than just one vector in the relevant subspace.
Indeed, recovering the full relevant subspace turns out to be significantly more challenging, and the core technical contribution of this work is to show how to do this.

\begin{remark}[Relation to Sliced Inverse Regression]\label{remark:slice}
	The trick of conditioning only on $(x,y)$ for which $\abs{y}$ is sufficiently large is reminiscent of the technique of \emph{slicing} originally introduced by \cite{li1991sliced} in the context of learning multi-index models. The high-level idea of slicing is that for any fixed value of $y$, the conditional law of 
	$x|F(x) = y$ is likely to be non-Gaussian in most directions $v \in V$, so in particular, $\condE{xx^{\top} - \Id|F(x) = y}$ should be nonzero, and its singular vectors will lie in $V$. This can be thought of as filtered PCA with the choice of function $\psi(z) = \bone{z = y}$. The first issue with using such an approach to get an actual learning algorithm is that $\Pr[x]{F(x) = y} = 0$ for any $y$, and the workaround in non-asymptotic analyses of sliced inverse regression \cite{babichev2018slice} is to estimate something like $\E[y]{\condE{xx^{\top} - \Id|F(x) = y}}$ instead. While finite sample estimators for such objects are known, the conditions under which this approach can provably recover the relevant subspace are quite strong and not applicable to our setting.
\end{remark}

\paragraph{Learning the Full Subspace: What Doesn't Work} 

One might hope that a more refined analysis shows that for a suitable $\tau$, the spectrum of $\vM_\tau$ can identify the entire subspace $V$.  
 Given that we can already learn some $w\in V$ with the threshold approach above, a first step would be to try to find a direction in $V$ orthogonal to $w$, by lower bounding the contribution to the Frobenius norm of $\vM_{\tau}$ from vectors orthogonal to $w$. Concretely, letting $\Pi_{V\backslash\brc{w}}$ denote the projector to the orthogonal complement of $w$ in $V$, we have that \begin{equation}
	\iprod{\Pi_{V\backslash\brc{w}}, \vM_{\tau}} = \E[x,y]*{\bone{\abs{y} > \tau}\cdot (\norm{\Pi_{V\backslash\brc{w}}x}^2 - (k-1))}. \label{eq:projectfirst}
\end{equation}
As before, if one could choose $\tau$ for which $\abs{F(x)} > \tau$ only if $\norm{\Pi_{V\backslash\brc{w}}x}^2 \ge k$, and if we could lower bound $\Pr{\abs{y} > \tau}$, then we would conclude that $\iprod{\Pi_{V\backslash\brc{w}}, \vM_{\tau}}\ge \Pr{\abs{y} > \tau}$, so $\vM_{\tau}$ has some other singular vector, orthogonal to $w$, with non-negligible singular value. The issue is that such a $\tau$ typically does not exist! For $x$ satisfying $\norm{\Pi_{V\backslash\brc{w}}x}^2 \le k$, $F(x)$ can be arbitrarily large, because $\norm{\Pi_w x}$ can be arbitrarily large.

It may be possible to lower bound the quantity in \eqref{lem:projecterror} using a more refined argument, but for general deep ReLU networks or piecewise linear functions, this seems very challenging. At the very least, one must be careful not to prove something too strong, like showing that $v^{\top}\vM_{\tau}v$ is non-negligible for \emph{any} unit vector $v\in V$. For instance, even when $L = 0$, it could be that all but one of the rows of $\vW_0$ lie in a proper subspace $W\subsetneq V$, and for the remaining row $u$ of $\vW_0$, $\norm{\Pi_{V\backslash W} u}/\norm{u}$ is arbitrarily small. In this case, for $v$ in the direction of $\Pi_{V\backslash W}u$, the quadratic form $v^{\top}\vM_{\tau}v$ is arbitrarily small, and it would be impossible to recover all of $V$ from a reasonable number of samples.

More generally, any proposed algorithm for learning all of $V$ had better be consistent with the fact that it is impossible to recover the full subspace $V$ within a reasonable number of samples if almost all of the variance of $F$ is explained by some proper subspace $W\subsetneq V$, or equivalently, if the ``leftover variance'' $\E[x]{(F(x) - F(\Pi_W x))^2}$ is negligible. We emphasize that this is a key subtlety that does not manifest in previous works that consider full-rank, well-conditioned weight matrices.

\paragraph{Learning the Full Subspace: Our Approach}

We now explain our approach. At a high level, we try to learn orthogonal directions inside the relevant subspace in an iterative fashion. The threshold filter approach above already gives us a single direction in $V$. Suppose inductively that we've learned some orthogonal vectors $w_1,...,w_{\ell}\in V$ spanning a subspace $W\subseteq V$ and want to learn another (note that technically we can only guarantee $w_1,...,w_{\ell}$ are approximately within $V$, but let us temporarily ignore this for the sake of exposition). Motivated by the above consideration regarding ``leftover variance,'' we proceed by a win-win argument: either the leftover variance already satisfies $\E[x]{(F(x) - F(\Pi_W x))^2} \le \epsilon$ in which case we are already done, or we can learn a new direction via the following crucial modification of the threshold filter.

First, as a thought experiment, consider the following matrix
\begin{equation}
    \vM^{W}_{\tau} \triangleq \Pi_{W^{\perp}}\E[x,y]*{\bone{\abs{y - F(\Pi_{W}x)} > \tau}\cdot (xx^{\top} - \Id)}\Pi_{W^{\perp}}.\label{eq:Ml_pre}
\end{equation}

Note the critical fact that we threshold on $y - F(\Pi_W x)$ as opposed to just on $y$. As before, it is not hard to show that if this matrix is nonzero, then its singular vectors with nonzero singular value must lie in $\vW_0$ and be orthogonal to $W$; thus giving us a new direction in $\vW_0$. We claim that if the leftover variance is non-negligible, then the above matrix will give us a new direction in $W$. 

The intuition behind the above matrix is as follows. Let $V\backslash W$ denote the subspace of $V$ orthogonal to $W$. We can write $F(x) = F(\Pi_V x) = F(\Pi_W x + \Pi_{V \backslash W} x)$. Now, as $F$ is Lipschitz, we can bound $G(x) = y - F(\Pi_W x) = F(\Pi_W x + \Pi_{V \backslash W} x) - F(\Pi_W x)$ as $|G(x)| \leq \Lambda \|\Pi_{V \backslash W} x\|^2$, where $\Lambda$ is the Lipschitz constant of $F$. In other words, $G(x)$ is bounded over $x$ for which $\|\Pi_{V\backslash W} x\|$ is bounded. Recall that the fact that $F(x)$ is not bounded over such $x$ was the key obstacle to using the original threshold filter approach to learn the full subspace.


The upshot is that for a suitably large $\tau$, the only contribution to the matrix $\vM_\tau^W$ should be from inputs $x$ that have large projection in $V\setminus W$. We are now in a position to adapt the analysis lower bounding $\iprod{\Pi_V,\vM_{\tau}}$ to lower bounding $\iprod{\Pi_{V\backslash W},\vM^W_{\tau}}$. In particular, we can apply the aforementioned anti-concentration for piecewise linear functions \emph{to the function $G$} and argue that, provided the leftover variance $\E[x]{(F(x) - F(\Pi_W x))^2} = \E[x]{G(x)^2}$ is non-negligible, the top singular vector of $\vM_\tau^W$ will give us a new vector in $V\setminus W$.

That being said, an obvious obstacle in implementing the above is that along with not knowing the true subspace $\vW_0$, we also don't know the true function $F$. This precludes us from forming the matrix $\vM_\tau^W$ as defined above. 

To get around this, we will enumerate over a sufficiently fine net of ReLU networks $\td{F}$ \emph{with relevant subspace $W$}, one of which will be close to the ReLU network $F(\Pi_W x)$. For each $\wt{F}$, we will form the matrix \begin{equation}
    \wt{\vM}^{W}_{\tau} \triangleq \Pi_{W^{\perp}}\E[x,y]*{\bone{\abs{y - \wt{F}(\Pi_{W}x)} > \tau}\cdot (xx^{\top} - \Id)}\Pi_{W^{\perp}}.\label{eq:Mlhat}
\end{equation}
and output the top singular vector as our new direction only if it has non-negligible singular value.

Arguing soundness, i.e. that this procedure doesn't yield a ``false positive'' in the form of an erroneous direction lying far from $V$, is not too hard. However, analyzing completeness, i.e. that this procedure will find \emph{some} new direction, is surprisingly subtle (see Lemma~\ref{lem:main_stable}). Formally, we need to argue that if we have an approximation $\wt{F}$ to the true $F$ (under some suitable metric), then the corresponding matrix $\td{\vM}^W_{\tau}$ is close to the matrix $\vM_\tau^W$. This is further complicated by the fact that ultimately, we will only have access to a subspace $W$ which is \emph{approximately} in $V$, as every direction we find in our iterative procedure is only guaranteed to \emph{mostly} lie within $V$.

Our key step in proving this is showing a new stability property of affine thresholds of piecewise linear functions and makes an intriguing connection to \emph{lattice polynomials} in tropical geometry.

\paragraph{Stability of Piecewise Linear Functions}

Following the above discussions, to complete our analysis we need to show \emph{stability} of affine thresholds of ReLU networks in the following sense: if $F, \tilde{F}:\R^d \to \R$ are two RELU networks that are close in some structural sense (i.e., under some parametrization), then $\E{\bone{|F(x)| > \tau} (xx^T - Id)} \approx \E{\bone{|\tilde{F}(x)|  > \tau} (xx^T - Id)}$. A natural way to approach the above is to upper bound $\Pr{\abs{F(x)}> \tau \wedge \abs{\td{F}(x)}\le \tau}$. That is, affine thresholds of ReLU networks that are structurally close disagree with low probability. 

A natural way to parametrize closeness is to require the weight matrices of the two networks $F, \tilde{F}$ to be close to each other. While such a statement is not too difficult to show for depth two networks (by a union bound over pairs of ReLUs), proving such a statement for general ReLU networks using a direct approach seems quite challenging. We instead look at proving such a statement for a more general class of functions - continuous piecewise-linear functions which allows us to do a certain kind of hybrid argument more naturally. 

Concretely, we show that affine thresholds of piecewise-linear functions that are close in some appropriate structural sense disagree with low probability over Gaussian space. We will elaborate upon the notion of structural closeness we consider momentarily, but for now it is helpful to keep in mind that it specializes to $L_2$ distance for linear functions. 


\begin{lemma}[Informal, see Lemma~\ref{lem:stability}]\label{lem:stability_informal}
Let $F,\td{F}: \R^d\to\R$ be piecewise-linear functions, both consisting of at most $m$ linear pieces, which are ``$(m,\eta)$-structurally-close'' (see Definition~\ref{def:closepiece}). For any $\tau > 0$,
\begin{equation}
\Pr[x\sim\N(0,\Id)]*{\abs{F(x)}> \tau \wedge \abs{\td{F}(x)}\le \tau}\le O(\eta m^2/\tau).\label{eq:stability_informal}
\end{equation}
\end{lemma}

To get a sense for this, suppose $F,\td{F}$ were even close in the sense that the polyhedral regions over which $F$ is linear are \emph{identical} to those over which $\td{F}$ is linear, and furthermore $\E[x]{(F(x) - \td{F}(x))^2}^{1/2} \le \eta$. Then if we take for granted that Lemma~\ref{lem:stability_informal} holds when $m = 1$, i.e. when $F,\td{F}$ are linear (see Lemma~\ref{lem:stability_ez}), it is not hard to show an $O((\eta m/\tau)^c)$ upper bound in \eqref{eq:stability_informal} under this very strong notion of closeness for some $c < 1$. Because $F$ and $\td{F}$ are $L_2$-close as functions, for any $t > 0$ we have that with probability $1 - O(\eta^2/t^2)$ the input $x\sim\calN(0,\Id)$ lies in a polyhedral region for which the corresponding linear functions for $F$ and $\td{F}$ are $t$-close. By the $m = 1$ case of Lemma~\ref{lem:stability_informal}, over any one of these at most $m$ regions, the affine thresholds $\bone{\abs{F(x)} > \tau}$ and $\bone{\abs{F(x)} > \tau}$ disagree with probability $O(t/\tau)$. Union bounding over these regions as well as the event of probability $\eta^2/t^2$ that $x$ does not fall in such a polyhedral region, we can upper-bound the left-hand side of \eqref{eq:stability_informal} by $O(\eta^2/t^2 + m t/\tau)$, and by taking $t = (\eta^2\tau/m)^{1/3}$, we get a bound of $(\eta m^2/\tau)^{2/3}$.

The issues with this are twofold. First, recall the function $\td{F}$ that we want to apply Lemma~\ref{lem:stability} to is obtained from some enumeration over a fine net of ReLU networks. As such there is no way to guarantee that the polyhedral regions defining $F$ and $\td{F}$ are exactly the same, making adapting the above argument far more difficult, especially for general ReLU networks. 

Second, we stress that the \emph{linear} scaling in $O(\eta)$ in \eqref{lem:stability_informal} is essential. If one suffered any polynomial loss in this bound as in the above argument, then upon applying Lemma~\ref{lem:stability_informal} $k$ times over the course of our iterative algorithm for recovering $V$, we would incur time and sample complexity \emph{doubly exponential} in $k$. The reason is as follows.

Recall that in the final argument we can only ensure that the directions $w_1,\ldots,w_{\ell}$ we have found so far are \emph{approximately} within $V$, and the parameter $\eta$ will end up scaling with an appropriate notion of subspace distance between $W$ and the true space $V$. On the other hand, the bound we can show on how far $\td{M}^W_{\tau}$ deviates from $M^W_{\tau}$ in spectral norm will essentially scale with the right-hand side of \eqref{lem:stability_informal}. So if we could only ensure $\td{M}^W_{\tau}$ and $M^W_{\tau}$ are $O(\eta^c)$-close in spectral norm for $c<1$, then if we append the top eigenvector of $\td{M}^W_{\tau}$ to the list of directions $w_1,...,w_{\ell}$ we have found so far, the resulting span will only be $O(\eta^c)$-close in subspace distance. Iterating, we would conclude that for the final output of the algorithm to be sufficiently accurate, we would need the error incurred by the very first direction $w_1$ found to be doubly exponentially small in $k$!

\paragraph{Lattice Polynomials}

It turns out that there is a clean workaround to both issues: passing to the \emph{lattice polynomial} representation for piecewise-linear functions. Specifically, we exploit the following powerful tool:

\begin{theorem}[\cite{ovchinnikov2002max}, Theorem 4.1; see Theorem~\ref{thm:tropical} below]\label{thm:tropical_informal}
	If $F$ is continuous piecewise-linear, there exist linear functions $\brc{g_i}_{i\in[M]}$ and subsets $\calI_1,...,\calI_{m}\subseteq [M]$ for which \begin{equation}\label{eq:lattice_pre}
    F(x) = \max_{j\in[m]}\min_{i\in\calI_j}g_i(x).
	\end{equation}
\end{theorem}
In fact, our notion of ``structural closeness'' will be built around this structural result. Roughly speaking, we say two piecewise linear functions are structurally close if they have lattice polynomial representations of the form \eqref{eq:lattice_pre} with the same set of clauses and whose corresponding linear functions are pairwise close in $L_2$ (see Definition~\ref{def:closepiece}).

At a high level, Theorem~\ref{thm:tropical_informal} will then allow us to implement a hybrid argument in the proof of Lemma~\ref{lem:stability_informal} and carefully track how the affine threshold computed by a piecewise-linear function changes as we interpolate between $F$ and $\td{F}$. In this way, we end up with the desired linear dependence on $\eta$ in \eqref{lem:stability_informal}.

With Lemma~\ref{lem:stability_informal} in hand, we can argue that even with only access to a subspace $W$ approximately within $V$ and with only a function $\td{F}$ that approximates $F(\Pi_W x)$, the top singular vector of \eqref{eq:Mlhat} mostly lies within $V$, and we can make progress.

Finally, we remark that as an added bonus, Theorem~\ref{thm:tropical} also gives us a way to enumerate over general continuous piecewise-linear functions! In this way, we can adapt our algorithm for learning ReLU networks to learning arbitrary piecewise-linear functions, with some additional computational overhead (see Theorem~\ref{thm:main_piecewise}).

\paragraph{Enumerating Over Piecewise-Linear Functions and ReLU Networks}

There is in fact one more subtlety to implementing the above approach for ReLU networks and getting singly exponential dependence on $k$.

First note that whereas one can always enumerate over functions computed by lattice polynomials of the form \eqref{eq:lattice_pre} in time $\exp(\poly(M))$ (see Lemma~\ref{lem:useorder}), for ReLU networks of size $S$ this can be as large as doubly exponential in $S$. Instead, we enumerate over ReLU networks in the naive way, that is, enumerating over the $\exp(O(S))$ many possible architectures and netting over weight matrices with respect to spectral norm, giving us only singly exponential dependence on $S$.

Here is the subtlety. Obviously two ReLU networks with the same architecture and whose weight matrices are pairwise close in spectral norm will be close in $L_2$. But how do we ensure that the corresponding lattice polynomials guaranteed by Theorem~\ref{thm:tropical_informal} are structurally close? In particular, getting anything quantitative would be a nightmare if the clause structure of these lattice polynomials depended in some sophisticated, possibly discontinuous fashion on the precise entries of the weight matrices.

Our workaround is to open up the black box of Theorem~\ref{thm:tropical_informal} and give a proof for the special case of ReLU networks from scratch. In doing so, we will find out that there are lattice polynomial representations for ReLU networks which only depend on the architecture and the \emph{signs} of the entries of the weight matrices (see Theorem~\ref{thm:tropical_nns}). In this way, we can guarantee that a moderately fine net will contain a network which is structurally close to the true network.


\section{Technical Preliminaries}

In this section we collect notation and technical tools that will be useful in the sequel.

\subsection{Miscellaneous Notation and Definitions}

We will use $\norm{\cdot}_p$ to denote the $L_p$ norm of a vector or of a random variable. When the random variable is given by a function over Gaussian space, e.g. $F(x)$ for $x\sim\calN(0,\Id)$ and $F:\R^d\to\R$, we use the short-hand $\norm{F}_p$ to denote $\E[x\sim\calN(0,\Id)]{F(x)^p}^{1/p}$. When $p = 2$, we will omit the subscript. We use $\norm{\cdot}_{\op}$ and $\norm{\cdot}_F$ to denote operator and Frobenius norms respectively. When we refer to a function as $\Lip$-Lipschitz, unless stated otherwise we mean with respect to $L_2$.

Given a subspace $V\subset\R^d$, let $\Pi_V$ denote the orthogonal projector to that subspace. Let $\S_V\subset\R^{d}$ denote the set of vectors in $V$ of unit norm. When the ambient space $\R^d$ is clear from context, we let $V^{\perp}$ denote the orthogonal complement of $V$. For a subspace $W\subseteq V$, we will denote the orthogonal complement of $W$ inside $V$ by $V\backslash W$.

Given $x\in\R$, let $\N(0,1,x)$ denote the standard Gaussian density's value at $x$. Let $\erfc(z) \triangleq \Pr[g\sim\N(0,1)]{\abs{g} > z}$ (note that we eschew the usual normalization). Let $\chi^2_m$ denote the chi-squared distribution with $m$ degrees of freedom.

Recall that we denote the ReLU activation function by $\phi(z) \triangleq \max(z,0)$. Additionally, for $\eta>0$, let $\clip_{\eta}:\R\to\R$ denote the function given by \begin{equation}
    \clip_{\eta}(z) =
        \begin{cases}
            z & \text{if} \ \abs{z} \le \eta \\
            0 & \text{otherwise}
        \end{cases}
\end{equation}
Overloading notation, given a vector $v\in\R^m$, we will use $\clip_{\eta}(v)$ to refer to the vector in $\R^m$ obtained by applying $\clip_{\eta}$ entrywise.

We will use the following basic property of the clipping operation:

\begin{fact}\label{fact:clipping}
    Suppose $v,v'\in\R^m$ satisfy $\norm{v - v'}_{\infty} \le \eta$, and define $v''\triangleq \clip_{\eta}(v')$. Then for any $i\in[m]$, $v_i v''_i \ge 0$.
\end{fact}

\begin{proof}
    If $v''_i > 0$, then $v''_i = v'_i >\eta$ and by triangle inequality, $v_i > 0$. Similarly, if $v''_i < 0$, then $v''_i = v'_i < -\eta$ and by triangle inequality, $v_i < 0$.
\end{proof}

Lastly, we will use $\vee$ and $\wedge$ to denote max and min respectively. The following class of functions will be useful for us.

\begin{definition}\label{defn:lattice}
    The set of \emph{lattice polynomials over the reals} is the set of real-valued functions defined inductively as follows: for any $d\ge 1$, any constant real-valued function $\R^d\to\R$ is a lattice polynomial, and any function $h:\R^d\to\R$ which can be written as $h(x) = \Max{f(x)}{g(x)}$ or $h(x) = \Min{f(x)}{g(x)}$ for two lattice polynomials $f,g:\R^d\to\R$ is also a lattice polynomial.
\end{definition}

\subsection{Concentration and Anti-Concentration}

\begin{fact}[Elementary anticoncentration]\label{fact:basicanti}
If $Z$ is a random variable for which $\abs{Z} \le M$ almost surely, and $\E{Z^2}\ge \sigma^2$, then $\Pr{\abs{Z} \ge t} \ge \frac{1}{M^2}(\sigma^2 - t^2)$.
\end{fact}

\begin{proof}
    We have \begin{align}
        \sigma^2\le \E*{Z^2} &= \E*{Z^2\mid \abs{Z} \ge t}\cdot\Pr*{\abs{Z} \ge t} + \E*{Z^2\mid \abs{Z} < t}\cdot\Pr*{\abs{Z} < t} \\
        &\le M^2\cdot\Pr*{\abs{Z} \ge t} + t^2,
    \end{align} from which the claimed bound follows upon rearranging.
\end{proof}

\begin{fact}\label{fact:erfbound}
For any integer $m\ge 1$ and $t \ge 0$, $\erfc(z) \ge \sqrt{2/\pi}\cdot \frac{t\cdot e^{-t^2/2}}{t^2+1}$.
\end{fact}

\begin{fact}\label{fact:convexity}
    The function $f:\R_{\ge 0}\to\R$ given by $f(z) = \erfc(1/\sqrt{z})\cdot z$ is convex over $\R_{\ge 0}$.
\end{fact}

\begin{proof}
    We can explicitly compute \begin{equation}
        f''(z) = \frac{e^{-1/2z} (1+z)}{2z^{5/2}\sqrt{2\pi}},
    \end{equation} which is clearly nonnegative for any $z\ge 0$.
\end{proof}

\begin{lemma}[\cite{vershynin2010introduction}]
\label{lem:matrixconcentration}
	Let $f:\R\to[0,1]$ be any function. Let $\vM = \E[x\sim\N(0,\Id_d)]{f(x)\cdot (xx^{\top} - \Id)}$. For any $\epsilon,\delta>0$, if $x_1,...,x_N\sim\N(0,\Id_d)$ for $N = \Omega\left(\frac{1}{\epsilon^2}(d +\log 1/\delta)\right)$, then \begin{equation}
		\Pr*{\norm*{\vM - \frac{1}{N}\sum_i f(x_i)\cdot (x_ix_i^{\top} - \Id)}_{\op} \ge \epsilon} \le\delta.
	\end{equation}
\end{lemma}

\begin{proof}
	This follows from standard sub-Gaussian concentration; see e.g. Remark 5.40 in \cite{vershynin2010introduction}.
\end{proof}

\begin{fact}[Sub-exponential tail bounds, see e.g. \cite{vershynin2010introduction}, Proposition 5.16]\label{fact:bernstein}
    If $X_1,...,X_N$ are i.i.d. random variables with mean zero and sub-exponential norm\footnote{Here we define the sub-exponential norm of a random variable $X$ to be $\sup_{p\ge 1}\frac{1}{p}\E{|X|^p}^{1/p}$} $K$, then \begin{equation}
        \Pr*{\abs*{\frac{1}{N}\sum^N_{i=1}X_i} \ge t} \le 2\exp\left(-\Omega\left(\Min{\frac{N t^2}{K^2}}{\frac{Nt}{K}}\right)\right).\label{eq:bernstein}
    \end{equation} In particular, for any $\delta > 0$, if we take $N = \Theta\left(\Max{\frac{K^2}{t^2}}{\frac{K}{t}}\right)\cdot\log 1/\delta$, then \eqref{eq:bernstein} is at most $\delta$.
\end{fact}

\begin{fact}[e.g. \cite{vershynin2018high}, Corollary 4.2.13]\label{fact:net}
    For any $\epsilon>0$, there is an $\epsilon$-net (in $L_2$ norm) of size $(1 + 2/\epsilon)^m$ for the unit $L_2$ ball in $m$ dimensions.
\end{fact}

\begin{corollary}\label{cor:opnet}
    For any $\epsilon,\beta>0$, there is an $\epsilon$-net (in operator norm) for the set of $m_1\times m_2$ matrices of operator norm at most $\beta$ of size at most $(1 + 2\beta/\epsilon)^{m_1m_2}$.
\end{corollary}

\begin{proof}
    As operator norm is upper bounded by Frobenius norm, an $\epsilon$-net in Frobenius norm for the set of $m_1\times m_2$ matrices of Frobenius norm at most $\beta$ would contain the claimed $\epsilon$-net. The former can be obtained from scaling an $\epsilon/\beta$-net in Frobenius norm for the set of $m_1\times m_2$ matrices of unit Frobenius norm, and such a net with size $(1 + 2\beta/\epsilon)^{m_1m_2}$ exists by Fact~\ref{fact:net}.
\end{proof}

\subsection{Power Method, Subspace Distances, and Perturbation Bounds}

\begin{fact}[Power method, see \cite{rokhlin2009randomized}]\label{thm:power-method}
Let $\vM \in \R^{d \times d}$, let $k \leq d$ be a non-negative integer, and let $\sigma_1 \geq \sigma_2 \geq \ldots \sigma_d$ be the nonzero singular values of $\vM$.
For any $k = 1, \ldots, d - 1$, let $\gap_k = \sigma_k / \sigma_{k + 1}$.
Suppose there is a \emph{matrix-vector oracle} which runs in time $R$, and which, given $v \in \R^d$, outputs $\vM v$.
Then, for any $\eta, \delta > 0$, there is an algorithm {\sc ApproxBlockSVD}$(\vM, \eta, \delta)$ which runs in time $\widetilde{O} (k R  \log \tfrac{1}{\eta \cdot \delta \cdot \gap_k})$, and with probability at least $1 - \delta$ outputs a matrix $\vec{U} \in R^{d \times k}$ with orthonormal columns so that $\|\vec{U} - \vec{U}_k \|_2 < \eta$, where $\vec{U}_k$ is the matrix whose columns are the top $k$ right singular vectors of $\vM$.
\end{fact}

\begin{lemma}[Gap-free Wedin, see \cite{allen2016lazysvd} Lemma B.3]\label{lem:wedin}
	Let $\epsilon,\xi,\mu > 0$. For symmetric matrices $\vec{A}, \wh{\vec{A}}\in\R^{d\times d}$ for which $\norm{\vec{A} -\wh{\vec{A}}}_{\op} \le \epsilon$, if $\wh{\vec{U}}$ is the matrix whose columns consist of the singular vectors of $\wh{\vec{A}}$ with singular value at least $\mu$, and $\vec{U}$ is the matrix whose columns consist of the singular vectors of $\vec{A}$ with singular value at most $\mu - \xi$, then $\norm{\wh{\vec{U}}^{\top}\vec{U}}_{\op} \le \epsilon/\xi$.
\end{lemma}

\begin{corollary}\label{cor:topsing}
    Let $\lambda \ge 2\epsilon > 0$. For symmetric matrices $\vec{A}, \wh{\vec{A}}\in\R^{d\times d}$ for which $\norm{\vec{A} -\wh{\vec{A}}}_{\op} \le \epsilon$ and $\norm{\wh{\vec{A}}}_{\op} \ge \lambda - \epsilon$, if $w\in\S^{d-1}$ is the top singular vector of $\wh{\vec{A}}$, and $V\subset\R^d$ is the orthogonal complement of the kernel of $\vec{A}$, then $\norm{\Pi_V w}_{\op} \ge 1 - 4\epsilon^2/\lambda^2$.
\end{corollary}

\begin{proof}
    If we take $\xi = \mu = \norm{\wh{\vec{A}}}$ in Lemma~\ref{lem:wedin}, then the columns of $\vec{U}$ (resp. $\wh{\vec{U}}$) in Lemma~\ref{lem:wedin} consist of an orthonormal basis $B\in\R^{d\times k}$ for the kernel of $\vec{A}$ (resp. $w$ and other singular vectors of $\vec{A}$, if any, with the same singular value), where $k$ is the dimension of $\ker(\vec{A})$. We have that \begin{equation}
        \norm{\Pi_{V^{\perp}}w}\le \norm{\wh{\vec{U}}^{\top}\vec{U}}_{\op} \le \epsilon/\norm{\wh{\vec{A}}}_{\op} \le \frac{\epsilon}{\lambda - \epsilon},
    \end{equation} from which we conclude that \begin{equation}
        \norm{\Pi_{V}w} \ge \left(1 - \left(\frac{\epsilon}{\lambda - \epsilon}\right)^2\right)^{1/2} \ge 1 - 4\epsilon^2/\lambda^2
    \end{equation} as claimed.
\end{proof}

\begin{definition}[Frames]
    A set of orthonormal vectors $\td{w}_1,...,\td{w}_{\ell}$ is a \emph{frame}. Given subspace $V\subset\R^d$, we say that this frame is \emph{$\nu$-nearly within $V$} if $\norm{\Pi_V \td{w}_i} \ge 1 - \nu$ for all $i$. We will sometimes refer to their span $\td{W}$ as a frame $\nu$-nearly within to $V$, when the choice of orthonormal basis for $\td{W}$ is clear from context.
\end{definition}

\begin{definition}[Subspace distances]
    Given $\ell$-dimensional subspaces $U_1,U_2\subset \R^d$, let $M_1,M_2\in\R^{d\times \ell}$ denote any two matrices whose columns consists of basis vectors for $U_1,U_2$ respectively. The \emph{chordal distance} $d_C(U_1,U_2)$ between $U_1$ and $U_2$ is defined by \begin{equation}
        d_C(U_1,U_2) = \left(\ell - \norm{M_1^{\top}M_2}^2_F\right)^{1/2}.
    \end{equation}
    The \emph{Procrustes distance} $d_P(U_1,U_2)$ between $U_1$ and $U_2$ is defined by \begin{equation}
        \inf_{\vec{O}\in O(r)} \norm{U_2 - U_1\cdot \vec{O}}_F,
    \end{equation} where $O(r)$ denotes the group of $r\times r$ orthogonal matrices.
\end{definition}

\begin{fact}[See e.g. \cite{chen2020learning}, Lemma 3.26]\label{fact:subspaceineq}
    Given $\ell$-dimensional subspaces $U_1,U_2\subset\R^d$, \begin{equation}
        d_P(U_1,U_2) \le \sqrt{2}d_C(U_1,U_2).
    \end{equation}
\end{fact}

\begin{lemma}\label{lem:frame}
    Let $\nu\le O(1/\ell^2)$. If $\Pi$ is an orthogonal projector to a subspace $V\subset\R^d$, and $\td{w}_1,...,\td{w}_{\ell}$ are a frame $\nu$-nearly within $V$, then there exists an orthonormal set of vectors $w_1,...,w_{\ell}$ spanning $W\subset V$ for which $d_C(\td{W},W) \le \sqrt{2\nu\cdot \ell}$ and $\norm{w_i - \td{w_i}} \le 2\sqrt{\nu\cdot \ell}$ for all $i\in[\ell]$.
\end{lemma}

\begin{proof}
    Let $\td{W}$ be the subspace spanned by $\td{w}_1,...,\td{w}_{\ell}$, and let $W$ be the subspace spanned by $\Pi_V \td{w}_1,...,\Pi_V \td{w}_{\ell}$. First note that because $\nu \le \frac{1}{2\ell^2}$, $\td{W}$ and $W$ have the same dimension, that is, $\Pi_V\td{w}_1...,\Pi_V\td{w}_{\ell}$ are linearly independent. Indeed, we have that $\iprod{\td{w}_i,\Pi_V \td{w}_i} \ge (1 - \nu)^2\ge 1 - 2\nu$, while $\iprod{\td{w}_i,\Pi_V \td{w}_j} = \iprod{\td{w}_i,\Pi_{V^{\perp}}\td{w}_j} \le (1 - (1 - \nu)^2)^{1/2} = \sqrt{2\nu}$ for $i\neq j$, so the Gram matrix of these vectors is diagonally dominant provided $\nu\le O(1/\ell^2)$.

    Overloading notation, let $W$ (resp. $\td{W}$) also denote the $d\times \ell$ matrices whose columns consist of some orthonormal basis vectors for $W$ (resp. the vectors $\td{w}_1,...,\td{w}_{\ell}$). The chordal distance $d_C(W,\td{W})$ satisfies \begin{equation}
        d_C(W,\td{W})^2 = \ell - \norm{W^{\top}\td{W}}^2_F = \ell - \sum \norm{\Pi_W \td{w}_i}^2 \le \ell - \ell\cdot(1 - \nu)^2 \le 2\nu\ell
    \end{equation} 
    Letting $O^* \triangleq \arg\inf_{O\in O(r)}\norm{W - O\td{W}}_F$ in the definition of $d_P(W,\td{W})$, we can take $w_1,...,w_{\ell}$ in the lemma statement to be the columns of $OW$. Then we have that $d_P(W,\td{W})^2 = \sum \norm{w_i - \td{w}_i}^2 \le 4\nu\cdot \ell$ by Fact~\ref{fact:subspaceineq}, from which the lemma follows.
\end{proof}

\begin{lemma}\label{lem:projecterror}
    For any $\vM\in\R^{d\times d}$ and a frame $\td{W}\in\R^{d\times \ell}$ which is $\nu$-nearly within an $\ell$-dimensional subspace $W$, we have that \begin{equation}
        \norm{\Pi_{\td{W}^{\perp}}\vM\Pi_{\td{W}^{\perp}} - \Pi_{W^{\perp}}\vM\Pi_{W^{\perp}}}_{\op} \le \sqrt{2}\cdot\norm{\vM}_{\op} \cdot d_C(\td{W},W).
    \end{equation}
\end{lemma}

\begin{proof}
    We bound $\norm{(\Pi_{\td{W}^{\perp}}-\Pi_{W^{\perp}})\vM\Pi_{\td{W}^{\perp}}}_{\op}$ and $\norm{\Pi_{W^{\perp}}\vM(\Pi_{\td{W}^{\perp}}-\Pi_{W^{\perp}})}_{\op}$ and apply triangle inequality. By sub-multiplicativity of the operator norm and the fact that projections have operator norm 1, $\norm{(\Pi_{\td{W}^{\perp}}-\Pi_{W^{\perp}})\vM\Pi_{\td{W}^{\perp}}}_{\op} \le \norm{\Pi_{\td{W}^{\perp}}-\Pi_{W^{\perp}}}_{\op}\cdot \norm{\vM}_{\op}$. Finally, note that
    \begin{equation}
        \norm{\Pi_{\td{W}^{\perp}}-\Pi_{W^{\perp}}}^2_2 \le \norm{\Pi_{\td{W}^{\perp}}-\Pi_{W^{\perp}}}^2_F = \norm{\Pi_{\td{W}}-\Pi_{W}}^2_F = 2(\ell - \iprod{\Pi_{\td{W}},\Pi_{W}}) = 2d_C(\td{W},W)^2,
    \end{equation} from which the claim follows.
\end{proof}


\section{Continuous Piecewise-Linear Functions and Lattice Polynomials}
\label{sec:polygonal_prelims}

In this section, we introduce tools for reasoning about continuous piecewise-linear functions, culminating in a structural result (Theorem~\ref{thm:tropical_nns}) giving an explicit representation of arbitrary ReLU networks as lattice polynomials (see Definition~\ref{defn:lattice}).

\subsection{Basic Notions}

We will work with functions which only depend on some low-dimensional projection of the input.

\begin{definition}[Subspace juntas]
A function $F:\R^d\to\R$ is a \emph{subspace junta} if there exist $v_1,...,v_k\in\S^{d-1}$ and a function $h:\R^k\to\R$ for which $F(x) = h(\iprod{v_1,x},...,\iprod{v_k,x})$ for all $x\in\R^d$. We will refer to $V \triangleq \Span(v_1,...,v_k)$ as the \emph{relevant subspace} of $F$, to $v_1,...,v_k$ as the \emph{relevant directions} of $F$, and to $h$ as the \emph{link function} of $F$.
\end{definition}

\begin{definition}[Piecewise Linear Functions]\label{def:piecewise}
Given vector space $W$, a function $h: W\to\R$ is said to be \emph{piecewise-linear} (resp. piecewise-affine-linear) if there exist finitely many linear (resp. affine linear) functions $\brc{g_i: W\to\R}_{i\in[M]}$ and a partition of $W$ into finitely many polyhedral cones $\brc{S_i}_{i\in\calI}$ such that $G(x) = \sum_i \bone{x\in S_i}g_i(x)$. We will say that $h$ is \emph{realized by $M$ pieces} $\brc{(g_i,S_i)}$ (note that $h$ can have infinitely many realizations). If each $g_i$ is given by $g_i(x) = \iprod{u_i,x} + b_i$ for some $u_i\in W, b_i\in\R$, then we will also refer to the pieces of $h$ by $\brc{(\iprod{u_i,\cdot} + b_i,S_i)}$.
\end{definition}

We are now ready to define the concept class we will work with in this paper.

\begin{definition}[``Kickers'']
We call a subspace junta $F$ with link function $h$ a \emph{kicker} if $h$ is continuous piecewise-linear. Note that a kicker is itself a continuous piecewise-linear function, and for any realization of its link function by $M$ pieces, there is a realization of $F$ by $M$ pieces.
\end{definition}

Henceforth, fix a subspace junta $F:\R^d\to\R$ with link function $h$ and relevant directions $v_1,...,v_k$ spanning relevant subspace $V\subset\R^d$.

\begin{example}[ReLU Networks]
Feedforward ReLU networks as defined in Definition~\ref{defn:relunets} are kickers with relevant subspace of dimension at most $k$, where $k$ is the row span of the weight matrix $\vW_0$, the link function is defined by \begin{equation}
    h(z) = \vW_{L+1}\phi(\vW_L\phi(\cdots \vW_1\phi(z)\cdots )),
\end{equation} and the pieces in one possible realization of $h$ correspond to the different possible sign patterns that the activations could take on, that is the different possible values of the vector \begin{equation}
    \brc*{\vW_a\phi(\vW_{a-1}\phi(\cdots \vW_1\phi(z)\cdots ))}_{0\le a\le L}\in\prod^L_{a=0} \{\pm 1\}^{k_a}
\end{equation} as $z$ ranges over $\R^k$.
\end{example}


\begin{lemma}
    If $F$ is a $\Lip$-Lipschitz kicker, then for any realization of its link function $h$ by pieces $\brc{(\iprod{w_i,\cdot},S_i)}$, there is a realization by pieces $\brc{(\iprod{w'_i,\cdot},S_i)}$ for which $\max_i \norm{g_i} \le L$.\label{lem:normbound}.
\end{lemma}

\begin{proof}
    Consider any piece $(\iprod{w_i,\cdot},S_i)$. If there is some $x\in S_i$ for which there exists a ball of nonzero radius $r$ around $x$ contained in $S_i$, then clearly $L \ge \norm{w_i}$: take $x$ and $x + r\cdot w_i$ and note that \begin{equation}L\ge \frac{F(x + r\cdot w_i) - F(x)}{\norm{(x + r\cdot w_i) - x}} = \frac{r\norm{w_i}^2}{r\norm{w_i}} = \norm{w_i}.\end{equation} If no such $x$ and ball exist, then $S_i$ is not full-dimensional and therefore contained in a hyperplane $W\subset V$. Then if we replace $(\iprod{w_i,\cdot},S_i)$ in the realization of $h$ with $(\iprod{\Pi_W w_i,\cdot}, S_i)$, this is still a realization of $h$. Again, it would suffice for there to exist a ball, now in the subspace $W$, of nonzero radius around some point in $S_i$. If this is not the case, then $S_i$ is not a full-dimensional subset of $W$ and thus lies in a codimension 1 subspace of $W$. Continuing thus, we eventually obtain some (possibly zero) vector $w'_i$ for which replacing $(\iprod{w_i,\cdot}, S_i)$ in the realization of $h$ with $(\iprod{w'_i,\cdot},S_i)$ still gives a realization of $h$, and furthermore $\norm{w'_i} \le L$.
\end{proof}

\begin{definition}[Restrictions]
Given any nonzero linear subspace $W\subseteq V$, let $F|_W: W\to\R$ denote the restriction of $F$ to the subspace $W$. By abuse of notation, we will sometimes also regard $F|_W$ as a function over $\R^{d}$ given by $F|_W(x) = F(\Pi_W x)$.
\end{definition}

One of the main properties of kickers that we exploit is \emph{positive homogeneity}:

\begin{fact}[Positive homogeneity]\label{fact:homogeneity}
For any $\lambda \ge 0$ and $x\in\R^k$, $F(\lambda\cdot x) = \lambda F(x)$.
\end{fact}


The following property of restrictions of Lipschitz functions will be important.

\begin{lemma}\label{lem:key}
For any nonzero linear subspace $W\subseteq V$, and $\Lip$-Lipschitz function $F:\R^d\to\R$, \begin{equation}
    \sup_{x: \norm{\Pi_{V\backslash W} x}\le 1} \abs*{F(x) - F(\Pi_W x)} \le \Lip.
\end{equation}
\end{lemma}

\begin{proof}
   Because $F(x) = F(\Pi_V x)$ and $F(\Pi_W x) = F(\Pi_W \Pi_V x)$, we may assume without loss of generality that $x\in V$. For any $x\in V$ for which $\norm{\Pi_{V\backslash W}x} \le 1$, we have that \begin{equation}
       \abs{F(x) - F(\Pi_W x)} \le \Lip\norm{x - \Pi_W x} = \Lip\norm{\Pi_{V\backslash W}x} \le \Lip,
   \end{equation}
   as claimed.
\end{proof}

\subsection{A Generic Lattice Polynomial Representation}

Essential to our analysis is the following structural result from \cite{ovchinnikov2002max} which says that, perhaps surprisingly, \emph{any} piecewise linear function can be expressed as a relatively simple lattice polynomial.

\begin{theorem}[\cite{ovchinnikov2002max}, Theorem 4.1]\label{thm:tropical}
If $h: \R^n\to\R$ is a continuous piecewise-linear function which has a realization by pieces $\brc{(g_i,S_i)}_{i\in[M]}$, there exists a collection of \emph{clauses} $\calI_1,...,\calI_{m}\subseteq [M]$ for which \begin{equation}\label{eq:lattice}
    h(x) = \max_{j\in[m]}\min_{i\in\calI_j}g_i(x)
\end{equation}
\end{theorem}



We will work with the following notion of approximation for such lattice polynomials:

\begin{definition}\label{def:closepiece}
    Two continuous piecewise-linear functions $G,\td{G}: \R^d\to\R$ are \emph{$(M,\eta)$-structurally-close} if there exist linear functions $g_1,...,g_M$ and $\td{g}_1,...,\td{g}_M$ and subsets $\calI_1,...,\calI_m\subseteq[M]$ for which \begin{equation}
        G(x) = \max_{j\in[m]}\min_{i\in\calI_i}g_i(x) \qquad \td{G}(x) = \max_{j\in[m]}\min_{i\in\calI_i}\td{g}_i(x)
    \end{equation}
    and $\norm{g_i - \td{g}_i} \le \eta$ for all $i$.
\end{definition}

Structural closeness of continuous piecewise-linear functions in the above sense is stronger than $L_2$-closeness. 

\begin{lemma}\label{lem:hybrid_L2}
Take continuous piecewise-linear functions $G,\td{G}:\R^m\to\R$ which are $(M,\eta)$-structurally-close. Then $\norm{G - \td{G}} \le \eta\sqrt{m}$. In particular, if $G$ is a piecewise-linear function which is realized by pieces $\brc{(\iprod{u_i,\cdot},S_i)}$ satisfying $\norm{u_i} \le \eta$, then $\norm{G} \le \eta\sqrt{m}$.
\end{lemma}

To show this, we need the following helper lemma:

\begin{lemma}\label{lem:hybrid}
    If $\brc{g_i}_{i\in[M]}$ and $\brc{\td{g}_i}_{i\in[M]}$ are two collections of linear functions, then for any $x$,
    \begin{equation}
        \abs{\max_{j\in[m]}\min_{i\in\calI_j}g_i(x) - \max_{j\in[m]}\min_{i\in\calI_j}\td{g}_i(x)} \le \max_i\abs{g_i(x) - \td{g}_i(x)} \label{eq:hybrid}
    \end{equation}
\end{lemma}

\begin{proof}
    This simply follows by induction using the fact that if $f_1,f_2:\R^a\to\R$ are both 1-Lipschitz with respect to $L_{\infty}$, then $\Max{f_1}{f_2}$ and $\Min{f_1}{f_2}$ are as well.
\end{proof}

\begin{proof}[Proof of Lemma~\ref{lem:hybrid_L2}]
    Let $\brc{(\iprod{u_i,\cdot},S_i)}_{i\in[M]}$ and $\brc{(\iprod{\td{u}_i,\cdot},S_i)}_{i\in[M]}$ be the realizations of $G,\td{G}$ for which $\norm{u_i - \td{u}_i} \le \eta$. By Lemma~\ref{lem:hybrid} applied to these pieces, together with Cauchy-Schwarz, for any $x$ we have that $\abs{G(x) - \td{G}(x)} \le \eta\norm{x}$. So $\norm{G - \td{G}} \le \eta\cdot\E{\norm{x}^2}^{1/2} = \eta\sqrt{m}$.
\end{proof}

As discussed in Section~\ref{sec:overview}, for our application to learning general kickers, we will leverage the lattice polynomial representation in Theorem~\ref{thm:tropical} to grid over piecewise-linear functions. Note that \emph{a priori}, even if we knew exactly the set of linear functions $\brc{g_i}_{i\in[M]}$ in a realization of a piecewise-linear function, enumerating over all lattice polynomials of the form \eqref{eq:lattice} would require time doubly exponential in $M$, as there are $2^M$ possible clauses $\calI_j$ and $2^{2^M}$ possible sets of clauses $\brc{\calI_j}$.

By being slightly more careful, we can enumerate over piecewise linear functions in time $\exp(\poly(M))$.

\begin{definition}
An \emph{order type on $n$ elements} is specified by a function $\omega:[n]\to[n]$ for which every element from 1 to $\max_i \omega(i)$ is present. We say that a set of $n$ real numbers $z_1,...,z_n$ has order type $\omega$ (denoted $\brc{z_1,...,z_n} \vdash \omega$ if $z_i = z_j$ (resp. $z_i > z_j$, $z_i < z_j$) if and only if $\omega(i) = \omega(j)$ (resp. $\omega(i) > \omega(j)$, $\omega(i) < \omega(j)$). Denote the set of order types on $n$ elements by $\Omega_n$. Note that any set of real numbers has exactly one order type.
\end{definition}

\begin{lemma}\label{lem:useorder}
    If $F$ has a realization by pieces $\brc{(g_i,S_i)}_{i\in[M]}$, then there is a function $A: \Omega_M\to [M]$ such that for any $x$, \begin{equation}
        F(x) = \sum_{\omega\in\Omega_M}\bone*{\brc{g_i(x)}_{i\in[M]}\vdash \omega} \cdot g_{A(\omega)}(x).\label{eq:useorder}
    \end{equation}
\end{lemma}

\begin{proof}
    Let $F(x) = \max_{j\in[m]}\min_{i\in\calI_j} g_i(x)$ be the max-min representation guaranteed by Theorem~\ref{thm:tropical}. This representation implies that for a fixed order type $\omega$, there is some index $i\in[M]$ for which $F(x) = g_i(x)$ for all $x$ satisfying $\brc{g_i(x)}_{i\in[M]}\vdash \omega$. This gives the desired mapping $A$.
\end{proof}

Note that the set of functions $A:\Omega_M\to[M]$ is only of size $(M!)^M \le M^{M^2}$, so by Lemma~\ref{lem:useorder}, to enumerate over piecewise-linear functions with $M$ pieces we can simply enumerate over linear functions $\brc{g_i}$ together with all possible functions $A$ (see Algorithm~\ref{alg:enumerate} below).

\subsection{Lattice Polynomials for ReLU Networks}

Here we give an explicit proof of Theorem~\ref{thm:tropical} in the special case of ReLU networks. We emphasize that the specific nature of the construction exhibited in this theorem will be important in the proof of our main result for learning ReLU networks, and that simply applying Theorem~\ref{thm:tropical} in a black-box fashion will not suffice for our purposes.

\begin{theorem}\label{thm:tropical_nns}
    If $F\in\calC_S$ is a ReLU network with weight matrices $\vW_0\in\R^{k_0\times d}, \vW_1\in\R^{k_1\times k_0},\ldots, \vW_{L}\in\R^{k_{L}\times k_{L-1}}, \vW_{L+1}\in\R^{1\times k_{L}}$, and if $F'$ is a ReLU network with the same architecture as $F$, with weight matrices $\vW'_0,...,\vW'_{L+1}$, such that \begin{equation}
        (\vW_a)_{i,j} \cdot (\vW'_a)_{i,j} \ge 0 \qquad \forall \ 0\le a\le L+1, (i,j)\in[k_a]\times[k_{a-1}],
    \end{equation}
    then there exist vectors $v_1,...,v_M, v'_1,...,v'_M$ and clauses $\calI_1,...,\calI_m\subseteq[M]$, where $M = 2^S$, for which \begin{align}
        F(x) &= \max_{j\in[m]}\min_{i\in\calI_j}\iprod{v_i,x} \\
        F'(x) &= \max_{j\in[m]}\min_{i\in\calI_j}\iprod{v'_i,x}.
    \end{align} Specifically, $v_1,...,v_M$ consist of all vectors of the form $\vW_{L+1}\Sig_{L}\vW_{L}\Sig_{L-1}\cdots \cdots \Sig_0\vW_0$ for diagonal matrices $\Sig_i\in\{0,1\}^{k_i\times k_i}$, and $v'_1,...,v'_M$ are defined analogously.
\end{theorem}


We prove Theorem~\ref{thm:tropical_nns} by induction by exhibiting max-min representations for ReLUs, scalings, and sums of max-min formulas. Let $G:\R^d\to\R$ be a piecewise-linear function given by $G(x) \triangleq \max_{j\in [m]}\min_{i\in\calI_j}\iprod{u_i,x}$ for some subsets $\brc{\calI_1,...,\calI_m}$ of $[M]$ and vectors $\brc{u_1,...,u_M}$ in $\R^d$.

\begin{lemma}\label{lem:phirep}
     Let $u_{M+1} = 0$ and let $\calI_{m+1} = \brc{M+1}$. Then for all $x\in\R^d$, \begin{equation}
        \phi(G(x)) = \max_{j\in[m+1]}\min_{i\in\calI_j}\iprod{u_i,x}.
    \end{equation}
\end{lemma}

\begin{proof}
    This is immediate from the definition of $\phi$.
\end{proof}

\begin{lemma}\label{lem:scalingrep}
    For any $\lambda\in\R$, there exist subsets $\brc{\calJ_1,...,\calJ_{m'}}$ of $[M]$ such that for all $x\in\R^d$,
    \begin{equation}
        \lambda G(x) = \max_{j\in[m']}\min_{i\in \calJ_j} \iprod{\lambda u_i,x}. \label{eq:negation_formula}
    \end{equation} Furthermore, these subsets only depend on $\calI_1,....,\calI_m$ and the sign of $\lambda$.
\end{lemma}

\begin{proof}
    For $\lambda > 0$, we have $\calJ_j = \calI_j$ for all $j$. So it remains to show the claim for $\lambda = -1$. We can write $-G(x)$ as $\min_{j\in[m]}\max_{i\in\calI_j} \iprod{u_i,x}$. This is a lattice polynomial over the reals, and any lattice polynomial over a distributive lattice can be written in disjunctive normal form as $\max_{j\in[m']}\min_{i\in\calJ_j} \iprod{u_i,x}$ for some subsets $\brc{\calJ_j}$ (see e.g. \cite[Section II.5, Lemma 3]{birkhoff1940lattice}), from which the claim follows.
\end{proof}

\begin{lemma}\label{lem:sumrep}
    For any $k'\in\N$ and $b\in[k']$, let $G_b(x) = \max_{j\in [m_b]}\min_{i\in\calI^b_j}\iprod{u^b_i,x}$ for some subsets $\brc{\calI^b_j}$ of $[M_b]$ and vectors $\brc{u^b_i}$ in $\R^d$. For all $x\in\R^d$, \begin{equation}
        \sum^{k'}_{b=1}G_b(x) = \max_{(j_1,\ldots,j_{k'})\in[m_1]\times\cdots \times [m_{k'}]} \ \min_{(i_1,\ldots, i_{k'})\in\calI_{j_1}\times\cdots \times \calI_{j_{k'}}}\iprod{u^1_{i_1} +\cdots + u^{k'}_{i_{k'}},x}.\label{eq:sum_formula}
    \end{equation}
\end{lemma}

\begin{proof}
    Take any $x\in\R^d$, and for $b\in[k']$ suppose that $G_b(x) = \iprod{u^b_{i^*_b},x}$ for some index $i^*_b\in[M]$. Note that for any $\calI^1_{j_1},\ldots,\calI^{k'}_{j_{k'}}$ containing $i^*_1,\ldots,i^*_{k'}$ respectively, \begin{equation}\min_{(i_1,\ldots,i_{k'})\in\calI^1_{j_1}\times\cdots\times\calI^{k'}_{j_{k'}}}\iprod{u^1_{i_1} + \cdots + u^{k'}_{i_{k'}},x} = \iprod{u^{k'}_{i^*_{k'}} + \cdots + u^{k'}_{i^*_{k'}},x}.\end{equation}
    This shows that the right-hand side of \eqref{eq:sum_formula} is lower bounded by the left-hand side.

    We now show the other direction. For any $i'_1,\ldots,i'_{k'}$ for which $\iprod{u^1_{i'_1} + \ldots + u^{k'}_{i'_{k'}},x} > G_1(x) + \cdots + G_{k'}(x)$, we must have $\iprod{u^b_{i'_b},x} > G_b(x)$ for some $b\in[k']$. In this case, we know that for every clause $\calI^b_{j_b}$ in $G_b$ which contains $i'_b$, there is some $i\in \calI^b_{j_b}$ for which $\iprod{u^b_i,x} < \iprod{u^b_{i'_b},x}$. So for any $\calI^1_{j_1}, \ldots,  \calI^{k'}_{j_{k'}}$ containing $i'_1,\ldots,i'_{k'}$ respectively, the corresponding clause on the right-hand side of \eqref{eq:sum_formula} satisfies $\min_{(i_1,\ldots,i_{L'})\in\calI^1_{j_1}\times\cdots\times\calI^{L'}_{j_{L'}}}\iprod{u^1_{i_1}+ \cdots + u^{L'}_{i_{L'}},x} < \iprod{u^1_{i'_1} + \cdots + u^{L'}_{i'_{L'}},x}$. This concludes the proof that the left-hand side of \eqref{eq:sum_formula} is upper bounded by the left-hand side.
\end{proof}

We can now prove Theorem~\ref{thm:tropical_nns}:

\begin{proof}
    The claim is trivially true for $L = -1$. Suppose inductively that for some layer $0\le a \le L$, we have that for all $b\in[k_a]$, if we denote \begin{align}
        F_{a,b} &\triangleq \vW^b_a\phi\left(\vW_{a-1}\phi\left(\cdots \phi\left(\vW_0 x\right)\right)\right)\\
        F'_{a,b} &\triangleq \vW'^b_a\phi\left(\vW'_{a-1}\phi\left(\cdots \phi\left(\vW'_0 x\right)\right)\right),
    \end{align} where $\vW^b_a$ denotes the $b$-th row of $\vW_a$, then $F_{a,b}$ and $F'_{a,b}$ can be expressed as max-min formulas $\max_{j\in[m_{a,b}]}\min_{i\in\calI^{a,b}_j}\iprod{v^{a,b}_i,\cdot}$ and $\max_{j\in[m_{a,b}]}\min_{i\in\calI^{a,b}_j}\iprod{v'^{a,b}_i,\cdot}$ for some clauses $\brc{\calI^{a,b}_j}$ and vectors $v^{a,b}_i, v'^{a,b}_i$ comprised respectively of vectors of the form $\vW^b_a\Sig_{a-1}\cdots\Sig_0\vW_0$ and $\vW'^b_a\Sig_{a-1}\cdots\Sig_0\vW'_0$ for all possible diagonal matrices $\Sig_i\in\brc{0,1}^{k_i\times k_i}$. Then for any $b\in[k_{a+1}]$, note that $F_{a+1,b} = \vW^b_{a+1}\phi(F_{a,1},...,F_{a,k_a})$ and $F'_{a+1,b} = \vW'^b_{a+1}\phi(F'_{a,1},...,F'_{a,k_a})$. By Lemma~\ref{lem:phirep} and Lemma~\ref{lem:scalingrep}, if the entries of $\vW^b_{a}$ and $\vW'^b_a$ are $w_1,...,w_{k_{a+1}}$ and $w'_1,...,w'_{k_{a+1}}$ respectively, then for every $b'\in[k_a]$, if $w_{b'}\cdot w'_{b'} \ge 0$, then there exist max-min representations for $w_{b'}\phi(F_{a,b'})$ and $w'_{b'}\phi(F_{a,b'})$ with the same set of clauses.

    Finally, by Lemma~\ref{lem:sumrep}, there exist max-min representations for the scalar-valued functions $F_{a+1,b} = \sum^{k_{a}}_{b'=1} w_{b'}\phi(F_{a,b'})$ and $F'_{a+1,b} = \sum^{k_a}_{b'=1} w'_{b'}\phi(F'_{a,b'})$ with the same set of clauses. And the vectors in this max-min representation consist of all vectors of the form $\vW^b_{a+1}\Sig_a\cdots\cdots \Sig_0\vW_0$ and $\vW'^b_{a+1}\Sig_a\cdots\cdots \Sig_0\vW'_0$ respectively for $\Sig_i\in\brc{0,1}^{k_i\times k_i}$. This completes the inductive step.
\end{proof}


\section{Filtered PCA}
\label{sec:filter}
 
In this section we prove our main results on learning kickers and ReLU networks. Throughout, we will make the following base assumption about the function $F$.

\begin{assumption}\label{assume:bounds}
$F$ is a kicker which is $\Lip$-Lipschitz for some $\Lip\ge 1$ and has at most $M$ pieces.
\end{assumption}

While our techniques are general enough to work under just this assumption, for our main application to learning ReLU networks (Definition~\ref{defn:relunets}), we can obtain improved runtime guarantees by making the following additional assumption on $F$.

\begin{assumption}\label{assume:bounds2}
    $F$ is computed by a size-$S$ ReLU network\footnote{Note that this implies $M\le 2^S$.} with depth $L+2$ and weight matrices $\vW_0\in\R^{k_0\times d},\ldots\vW_L\in\R^{k_L\times k_{L-1}}, \vW_{L+1}\in\R^{1\times k_L}$ satisfying $\norm{\vW_i}_{\op}\le \normbound$ for all $0\le i \le L+1$, for some $\normbound\ge 1$.\footnote{Recall from Definition~\ref{defn:relunets} that we will refer to the rank of $\vW_0$ as $k$ to emphasize that $F$ is a kicker with relevant subspace $V$ of dimension $k$.}
\end{assumption}

In this section, unless stated otherwise, we will only assume $F$ satisfies Assumption~\ref{assume:bounds}, but in certain parts of the proof (e.g. Section~\ref{sec:nettingnets}), we will get better bounds by additionally making Assumption~\ref{assume:bounds2}. Formally, our main results are the following:

\begin{theorem}\label{thm:main_piecewise}
    Given access to samples from the distribution $\calD$ corresponding to kicker $F$ satisfying Assumption~\ref{assume:bounds}, {\sc FilteredPCA}($\calD,\epsilon,\delta$) outputs a kicker $\td{F}$ for which $\E{(y - \td{F}(x))^2}\le \epsilon^2$ with probability at least $1 - \delta$. Furthermore, {\sc FilteredPCA} has sample complexity \begin{equation}
        d\log(1/\delta)\cdot \poly\left(\exp\left(k^3\Lip^2/\epsilon^2\right),M^k\right)
    \end{equation} and runtime \begin{equation}
        \td{O}(d^2\log(1/\delta))\cdot M^{M^2}\cdot \poly\left(\exp\left(k^4\Lip^2/\epsilon^2\right),M^{k^2}\right).
    \end{equation}
\end{theorem}

\begin{theorem}\label{thm:main_nets}
    Given access to samples from the distribution $\calD$ corresponding to feedforward ReLU network $F$ satisfying Assumption~\ref{assume:bounds2}, {\sc FilteredPCA}($\calD,\epsilon,\delta$) outputs a ReLU network $\td{F}$ for which $\E{(y - \td{F}(x))^2}\le \epsilon^2$ with probability at least $1 - \delta$. Furthermore, {\sc FilteredPCA} has sample complexity \begin{equation}
        d\log(1/\delta) \poly\left(\exp\left(k^3\Lip^2/\epsilon^2\right),2^{kS},\left(\normbound^{(L+2)}/\Lip\right)^k\right)
    \end{equation} and runtime \begin{equation}
        \td{O}(d^2\log(1/\delta))\cdot \poly\left(\exp\left(k^3S^2\Lip^2/\epsilon^2\right),2^{k S^3},\left(\normbound^{L+2}/\Lip\right)^{kS^2}\right).
    \end{equation}
\end{theorem}

\begin{remark}[Scale Invariance]\label{remark:scaleinvariant}
    Often, guarantees for PAC learning ReLU networks are stated scale-invariantly in terms of the relative error $\E{(y - \td{F}(x))^2}/\E{y^2}$, or equivalently the absolute error $\E{(y - \td{F}(x))^2}$ for the true $F$ satisfying $\E{y^2} = 1$.

    In our general setting, recall from Example~\ref{example:spike} that some dependence on the Lipschitz constant of $F$ is needed. One standard way to achieve this is to normalize the weight matrices of the true underlying network $F$ to have operator norm at most $\normbound$, in which case the Lipschitz constant of $F$ is at most $\normbound^{L + 2}$ and, with our techniques, we can obtain guarantees depending just on $\normbound$ by using Theorem~\ref{thm:main_piecewise}. To obtain improved guarantees, we can additionally assume a better bound of $\Lip$ on the Lipschitz constant, and this gives rise to Theorem~\ref{thm:main_nets} above.

    Under this normalization in terms of $\Lip$ and $\normbound$, note that the sample complexity and runtime in Theorem~\ref{thm:main_nets} are scale invariant as the quantities $\Lip/\epsilon$ and $\normbound^{L+2}/\Lip$ are invariant under arbitrary rescalings of the $L + 2$ weight matrices of $F$. Also note that $\Lip$ can be any \emph{upper bound} on the actual Lipschitz constant of $F$, that is, the runtime guarantee in Theorem~\ref{thm:main_nets} does not degrade with the actual Lipschitz constant of $F$.
\end{remark}

In Section~\ref{sec:anti}, we prove an anti-concentration result for piecewise-linear functions. We use this in Section~\ref{sec:idealized} to prove that in an idealized scenario where we had exact access to some $\ell$-dimensional $W\subset V$ as well as exact query access to $F|_W$, we would be able to approximately recover a vector in $V\backslash W$ by running one iteration of the main loop of {\sc FilteredPCA}. In the remaining sections, we show how to pass from this idealized scenario to the setting we actually care about, in which we only samples $(x,F(x))$. In Section~\ref{sec:stability} we show that affine thresholds of piecewise-linear functions are stable under small perturbations of the function. Then in Section~\ref{sec:nettingjuntas}, we show how to grid over the set of kickers, and in Section~\ref{sec:nettingnets} we show how to grid over ReLU networks more efficiently and formally state our algorithm. In Section~\ref{sec:perturb} we combine these ingredients to argue that as long as we have sufficiently good approximate access to $W$ and $F|_W$, a single iteration of the main loop of {\sc FilteredPCA} will approximately recover a vector from $V\backslash W$. Lastly, in Section~\ref{sec:putting} we conclude the proofs of Theorem~\ref{thm:main_piecewise} and \ref{thm:main_nets}. At the very end, we discuss briefly why merely adapting the approach of \cite{chen2020learning} does not work.

\subsection{Anti-Concentration of Piecewise Linear Functions}
\label{sec:anti}

In this section, we show that for any continuous piecewise-linear function with some variance, the probability that it exceeds any given threshold is non-negligible.

\begin{lemma}\label{lem:anti1}
If $G: \R^m\to\R$ is continuous piecewise-linear and $\Lambda$-Lipschitz and $\E{G^2} \ge \sigma^2$, then for any $s \ge 0$, \begin{equation}
\Pr{|G| > s} \ge \Omega(\exp(-3ms^2/\sigma^2))\cdot \frac{s\sigma}{\sqrt{m}\Lambda^2}.\label{eq:anti_ourscaling1}
\end{equation}
\end{lemma}

\begin{proof}
    Let $\brc{(g_i,S_i)}$ be the pieces of some realization $G$, and for every $i$ let $u_i\in\R^m$ be the vector for which $g_i(\cdot) = \iprod{u_i,\cdot}$. By Lemma~\ref{lem:normbound}, we can assume $\norm{u_i} \le \Lambda$ for all $i$.
    
    Take any $i$ and define \begin{equation}\sigma^2_i \triangleq \E[x\sim\N(0,\Id)]{\iprod{u_i,x}^2 \mid x\in S_i}\end{equation} Note that if $i$ is chosen with probability $\Pr{x\in S_i}$, then $\E[i]{\sigma^2_i} \ge \sigma^2$.
    Because each $S_i$ is a polyhedral cone, sampling $x\sim\N(0,\Id)$ conditioned on $x\in S_i$ is equivalent to sampling $r\sim\chi^2_m$, independently sampling $\wh{x}\sim\S^{m-1}$ conditioned on $\wh{x}\in S_i$, and outputting  $r^{1/2}\cdot \wh{x}$. It follows that
    \begin{equation}
        \sigma^2_i = \E[r\sim\chi^2_m,\wh{x}\sim\S^{m-1}]{r\cdot \iprod{u_i,\wh{x}}^2 \mid \wh{x}\in S_i} = \E[r\sim\chi^2_m]{r} \cdot \E[\wh{x}\sim\S^{m-1}]{\iprod{u_i,\wh{x}}^2 \mid \wh{x} \in S_i} = m\cdot \E[\wh{x}\sim\S^{m-1}]{\iprod{u_i,\wh{x}}^2 \mid \wh{x} \in S_i}.
    \end{equation} By Fact~\ref{fact:basicanti}, $\Pr{\abs{\iprod{u_i,\wh{x}}} \ge \sigma_i/\sqrt{2m} \mid \wh{x}\in S_i} \ge \frac{\sigma^2_i}{2m\norm{u_i}^2}$. We conclude that for any $s> 0$, \begin{align}
        \Pr*{\abs{\iprod{u_i,x}} \ge s\mid x\in S_i} &\ge \Pr[r\sim\chi^2_m]*{r > 2m s^2/\sigma_i^2}\cdot \frac{\sigma^2_i}{2m\norm{u_i}^2} \\
        &\ge \erfc(s\sqrt{2m}/\sigma_i)\cdot \frac{\sigma^2_i}{2m\Lambda^2} 
        \label{eq:tailuix}
    \end{align} 
    
    
    By Fact~\ref{fact:convexity}, the right-hand side of \eqref{eq:tailuix} is convex as a function of $\sigma^2_i$, so \begin{align}
        \Pr{\abs{G(x)} > s} &\ge \E[i]*{\erfc(s\sqrt{2m}/\sigma_i)\cdot \frac{\sigma^2_i}{2m\Lambda^2}} \\
        &\ge \erfc(s\sqrt{2m}/\E[i]{\sigma^2_i}^{1/2})\cdot \frac{\E[i]{\sigma^2_i}}{2m\Lambda^2} \\
        &\ge \erfc(s\sqrt{2m}/\sigma)\cdot \frac{\sigma^2}{2m\Lambda^2} \\
        &\ge \sqrt{2/\pi}\cdot \frac{s\sqrt{2m}\cdot \exp({-ms^2/\sigma^2})}{\sigma\cdot (2ms^2/\sigma^2 + 1)}\cdot \frac{\sigma^2}{2m\Lambda^2} \\
        &\ge \Omega(\exp({-3ms^2/\sigma^2}))\cdot \frac{s\sigma}{\sqrt{m}\Lambda^2},
    \end{align} where the second step follows by Jensen's and the fourth step follows by Fact~\ref{fact:erfbound}.
\end{proof}

\subsection{An Idealized Calculation}
\label{sec:idealized}

Suppose we had access to an orthonormal collection of vectors $w_1,\ldots,w_{\ell}$ that are \emph{exactly} in $V$. Let $W$ denote their span. Suppose further that we had access to the matrix
\begin{equation}
    \vM^{W}_{\tau} \triangleq \Pi_{W^{\perp}}\E[x,y]*{\bone{\abs{y - F(\Pi_{W}x)} > \tau}\cdot (xx^{\top} - \Id)}\Pi_{W^{\perp}}.\label{eq:Ml}
\end{equation}
When the threshold $\tau$ is clear from context, we will just refer to this matrix as $\vM^W$.

As we will see, if this matrix is nonzero, then its singular vectors with nonzero singular value must lie in $V$ and be orthogonal to $w_1,\ldots,w_{\ell}$. The main challenge will be to show that this matrix is nonzero. The following proof also applies to the case of $\ell = 0$, in which case $F(\Pi_W x)$ specializes to the zero function and \eqref{eq:M0} specializes to
\begin{equation}
    \vM^{\emptyset}_{\tau} \triangleq \E[x,y]*{\bone{\abs{y} > \tau}\cdot (xx^{\top} - \Id)}. \label{eq:M0}
\end{equation}
In particular, \eqref{eq:M0} is a matrix we actually have access to at the beginning of the algorithm, and one consequence of the warmup argument below is an algorithm for finding a single vector in $V$.

We first show that for appropriately chosen $\tau$, either the top singular value of $\vM^{W}_{\tau}$ is non-negligible, or $\E{(F(x) - F(\Pi_W x)^2}$ is small, that is, $F$ is already sufficiently well-approximated by the function $F|_W$.

\begin{lemma}\label{lem:idealized_find}
Suppose $\E[x\sim\N(0,\Id)]{(F(x) - F(\Pi_{W}x))^2} \ge \rho^2$ for some $\rho > 0$. For any $\tau > 0$, if a vector is not in the kernel of $\vM^W_{\tau}$, then it must lie in $V\backslash W$.
For $\tau \ge \sqrt{2(k-\ell)}\cdot \Lip$, \begin{equation}
        \iprod*{\vM^W_{\tau},\Pi_{V\backslash W}} \ge \Omega\left(e^{-3k\tau^2/\rho^2}\right)\cdot \frac{(k - \ell)\tau\rho}{\sqrt{k}\Lip^2}.\label{eq:nextdirection}
    \end{equation}
    In particular, for this choice of $\tau$, the top singular vector of $\vM^W_{\tau}$ lies in $V\backslash W$ and has singular value at least $\lambda^{(\ell)}_{\tau}\triangleq \Omega\left(e^{-3k\tau^2/\rho^2}\right)\cdot \frac{\tau\rho}{\sqrt{k}\Lip^2}$.
\end{lemma}

\begin{proof}
The first part just follows from the fact that any $u\in \Pi_W$ is clearly in the kernel, and for any $u\in\S^{d-1}$ orthogonal to $V$, $\iprod{u,x}$ and $F(x)$ are independent, so \begin{equation}
    u^{\top}\vM^W_{\tau}u = \E[g\sim\N(0,1)]{g^2 - 1}\cdot \E[x]{\bone{\abs{F(x) - F(\Pi_W v)}>\tau}} = 0.
\end{equation}
For \eqref{eq:nextdirection}, we would like to apply Lemmas~\ref{lem:key} and \ref{lem:anti1} to the continuous piecewise-linear function $G(x) \triangleq F(x) - F(\Pi_W x)$. Pick an orthonormal basis $w_{\ell+1},\ldots,w_k$ for $V\backslash W$. For any $x$ for which $\norm{\Pi_{V\backslash W}x}\le 1$, Lemma~\ref{lem:key} implies $\abs{G(x)} \le \Lip$. So by positive homogeneity (see Fact~\ref{fact:homogeneity}) of $G(x)$ and the definition of $\tau$, $\abs{G(x)} > \tau$ only if $\norm{\Pi_{V\backslash W}x}^2 \ge 2(k-\ell)$, so
\begin{align}
    \sum^k_{i = \ell+1} w^{\top}_i \vM^W_{\tau} w_i &= \E[x]*{\bone{\abs{G(x)}> \tau}\cdot \left(\norm{\Pi_{V\backslash W}x}^2 - (k-\ell)\right)} \\
    &\ge (k-\ell)\cdot \Pr[x]*{G(x) > \tau}.
\end{align}
\eqref{eq:nextdirection} then follows from Lemma~\ref{lem:anti1} applied to $G$.

The final statement in Lemma~\ref{lem:idealized_find} follows by averaging.
\end{proof}


If $\epsilon$ is the target $L_2$ error to which we want to learn $F$, we will only ever work with $\rho \ge \Omega(\epsilon)$. In the sequel, we will take \begin{equation}
    \tau = c\sqrt{k}\cdot \Lip\label{eq:taudef}
\end{equation} for sufficiently large absolute constant $c>0$. As a result, we have that
\begin{equation}
    \lambda^{(\ell)}_{\tau} \ge \Omega\left(e^{-O(k^2 \Lip^2/\epsilon^2)}\right)\cdot (\epsilon/\Lip) \triangleq \ulam.\label{eq:lambound}
\end{equation}

\subsection{Stability of Piecewise Linear Threshold Functions}
\label{sec:stability}

To get an iterative algorithm for finding all relevant directions of $F$, we need to show an analogue of Lemma~\ref{eq:nextdirection} in the setting when we only have access to directions $\td{w}_1,\ldots,\td{w}_{\ell}$ which are \emph{close} to the span of $V$, and when we only have access to an \emph{approximation} of the function $F|_W$. 

In this section, we show the following stability result for affine thresholds of piecewise-linear functions:

\begin{lemma}\label{lem:stability}
Let $f,g,g': \R^d\to\R$ be piecewise-linear functions. For any $\tau > 0$, if $g,g'$ are $(m,\eta)$-structurally-close and $f$ has a realization with at most $m$ pieces, then
\begin{equation}
\Pr[x\sim\N(0,\Id)]*{\abs{g(x)-f(x)}> \tau \wedge \abs{g'(x) - f(x)}\le \tau}\le 9\eta m^2/\tau
\label{eq:disagree_g}
\end{equation}
\end{lemma}

An important building block of the proof is the special case where $f = 0$ and $g,g'$ are linear:

\begin{lemma}\label{lem:stability_ez}
For $\tau > 0$ and vectors $v,v'\in\R^d$,
\begin{equation}
    \Pr[x\sim\N(0,\Id)]{\iprod{v,x} > \tau \wedge \iprod{v',x}\le \tau} \le O\left(\frac{\norm{v - v'}}{\tau}\right)
    \label{eq:disagree}
\end{equation}
\end{lemma}

\begin{proof}
    First note that without loss of generality, we may assume that $\norm{v} \ge \norm{v'}$; if not, then the random variable $\bone{\iprod{v,x}>\tau \wedge \iprod{v',x}\le\tau}$ is stochastically dominated by $\bone{\iprod{ v,x}>\tau \wedge \iprod{\zeta v',x}\le\tau}$ for $\zeta = \norm{v}/\norm{v'}$, and furthermore $\norm{v - \zeta v'} \le \norm{v - v'}$ by the Pythagorean theorem.
    
    Also note that we may assume $\norm{v'} > \norm{v - v'}$. Otherwise, we would have $\norm{v} \le 2\norm{v - v'}$. But then we could upper bound the left-hand side of \eqref{eq:disagree} by \begin{equation}
        \Pr{\iprod{v,x} > \tau} \le e^{-\tau^2/2\norm{v}^2}\le e^{-\frac{\tau^2}{8\norm{v - v'}^2}}\le 2\norm{v - v'}/\tau.
    \end{equation}
    Now define $\wh{v} = v/\norm{v}$ and $\wh{v'} = v'/\norm{v'}$ so that \eqref{eq:disagree} equals $\Pr{\iprod{\wh{v},x} > \wh{\tau} \wedge \iprod{\wh{v'},x} \le \wh{\tau}'}$ for $\wh{\tau} \triangleq \tau/\norm{v}$ and $\wh{\tau}' \triangleq \tau/\norm{v'}$. Write $\wh{v}' = \alpha \wh{v} + \sqrt{1 - \alpha^2}v^{\perp}$ for $v^{\perp}$ orthogonal to $\wh{v}$, and denote the random variables $\iprod{\wh{v},x}$ and $\iprod{\wh{v}',x}$ by $\gamma$ and $\gamma'$ respectively (these are $\alpha$-correlated standard Gaussians).
    
    Note that by the assumption that $\norm{v} \ge \norm{v'} \ge \norm{v - v'}$, the angle between $v$ and $v'$ is at most $\pi/3$, so $\alpha \ge 1/2$.
    
    We are now ready to upper bound \eqref{eq:disagree}. We will split into two cases, either $\gamma > \wh{\tau}'/\alpha$ or $\wh{\tau}\le \gamma\le \wh{\tau}'$, and upper bound the contribution of either case to the probability in \eqref{eq:disagree} by $O({\norm{v - v'}}/{\tau})$, from which the lemma will follow.
        
    \noindent\textbf{Case 1}: $\gamma > \wh{\tau}'/\alpha$.

        The density of $\gamma'$ relative to $\gamma$ is given by \begin{equation}
            \int^{\frac{\wh{\tau}' - \alpha \gamma}{\sqrt{1 - \alpha^2}}}_{-\infty} \N(0,1,x) dx = \frac{1}{2}\erfc\left(\frac{\alpha \gamma - \wh{\tau}'}{\sqrt{1 - \alpha^2}}\right) \le \frac{1}{2}\exp\left(-\frac{(\alpha \gamma - \wh{\tau}')^2}{2(1 - \alpha^2)}\right).
            \label{eq:expbound}
        \end{equation}
        We have that
        \begin{align}
            \E[\gamma]*{\frac{1}{2}\exp\left(-\frac{(\alpha \gamma - \wh{\tau}')^2}{2(1 - \alpha^2)}\right) \cdot \bone{\gamma > \wh{\tau}'}} &= \frac{1}{4}\sqrt{1 - \alpha^2}\cdot \exp(-\wh{\tau}^{\prime 2}/2)\cdot \erfc(\wh{\tau}'\sqrt{1 - \alpha^2}/\alpha) \\
            &\le \frac{1}{4}\sqrt{1- \alpha^2}\cdot \exp(-\wh{\tau}^{\prime 2}/2\alpha^2) \\
            &\le \frac{\norm{v - v'}}{4\sqrt{2}\norm{v'}} \cdot \frac{\abs{\alpha}\sqrt{2}}{\wh{\tau}'} \le \frac{\norm{v - v'}}{4\tau},
        \end{align} where the first step is standard Gaussian integration, the second step uses the inequality $\erfc(z) \le e^{-z^2/2}$ for all $z\ge 0$, and the third step uses the fact that $\exp(-x) \le 1/x$ for all $x> 0$ and the fact that $\sqrt{1 - \alpha^2} = \frac{1}{\sqrt{2}}\norm{\wh{v} - \wh{v}'}\le \frac{\norm{v - v'}}{\sqrt{2}\norm{v'}}$.

    \noindent\textbf{Case 2}: $\wh{\tau}< \gamma \le \wh{\tau}'/\alpha$.
        
        We can naively upper bound the probability $\wh{\tau}< \gamma \le \wh{\tau}'/\alpha$ and $\gamma' \le \wh{\tau}'$ by the probability $\wh{\tau}< \gamma \le \wh{\tau}'/\alpha$, which is at most $e^{-\wh{\tau}^2/2}\cdot\left(\wh{\tau}'/\alpha - \wh{\tau}\right)$. Note that \begin{equation}
            \wh{\tau}'/\alpha - \wh{\tau} \le \tau\cdot\left(\frac{1/\alpha}{\norm{v'}} - \frac{1}{\norm{v}' + \norm{v-v'}}\right) \le \frac{\tau}{\alpha}\cdot\frac{(1-\alpha)\norm{v'} + \norm{v-v'}}{\norm{v'}^2} \le \frac{3\tau\norm{v - v'}}{2\alpha\norm{v'}^2},\label{eq:naive}
        \end{equation} where in the last step we have used that $1 - \alpha = \frac{1}{2}\norm{\wh{v} - \wh{v}'} \le \frac{\norm{v - v'}}{2\norm{v'}}$.
        
        Suppose to the contrary that $e^{-\wh{\tau}^2/2}\cdot (\wh{\tau}'/\alpha - \wh{\tau}) > \frac{9\norm{v - v'}}{\tau}$ so that by \eqref{eq:naive},
        \begin{equation}
            e^{\wh{\tau}^2/2} < \frac{\tau^2}{6\alpha\norm{v'}^2}.\label{eq:supposeforcontradict}
        \end{equation}
        
        Recall that we may assume that $\norm{v'} \ge \norm{v - v'}$, so $\wh{\tau}\ge \frac{\tau}{2\norm{v'}}$, and that $\alpha \ge 1/2$. From this, \eqref{eq:supposeforcontradict} would imply that $e^{\frac{\tau^2}{8\norm{v'}^2}} < \frac{\tau^2}{3\norm{v'}^2}$, and such an inequality cannot hold.
\end{proof}

We are now ready to prove Lemma~\ref{lem:stability}.

\begin{proof}[Proof of Lemma~\ref{lem:stability}]
    The left-hand side of \eqref{eq:disagree_g} is at most \begin{equation}
        \Pr[x\sim\N(0,\Id)]{g(x) - f(x) > \tau \wedge g'(x) - f(x) \le \tau} + \Pr[x\sim\N(0,\Id)]{g(x) - f(x) < -\tau \wedge g'(x) - f(x) \ge -\tau}, \label{eq:disagree2}
    \end{equation} and by symmetry it suffices to upper bound the former probability on the right-hand side of \eqref{eq:disagree2} by $O({\eta m^2}/{\tau})$.

    By definition of $(m,\eta)$-structural-closeness, we can express $g$ and $g'$ as $\max_j\min_{i\in\calI_j}\iprod{u_i,\cdot}$ and $\max_j\min_{i\in\calI_j}\iprod{u'_i,\cdot}$ respectively, for vectors $\brc{u_i}_{i\in[m]}$ and $\brc{u'_i}_{i\in[m]}$ for which $\norm{u_i - u'_i} \le \eta$ for all $i$.

    We proceed via a hybrid argument. Take any $0\le i\le m$. Let $u^{(i)}_1,\ldots,u^{(i)}_{i-1}$ be $u_1,\ldots,u_{i-1}$, and let $u^{(i)}_i,\ldots,u^{(i)}_m$ be the vectors $u'_i,\ldots,u'_m$. Define the function $g^{(i)} = \max_a\min_{b\in\calI_a}\iprod{u^{(i)}_i,x}$ so that $g^{(0)}(x) = \max_a \min_{b\in\calI_a}\iprod{u'_b,x}$ and $g^{(m)}(x) = \max_a\min_{b\in\calI_a}\iprod{u_b,x}$.


    We claim that for any $x$, $g^{(i-1)}(x)$ and $g^{(i)}(x)$ are sandwiched between $\iprod{u'_i,x}$ and $\iprod{u_i,x}$, in the sense that \begin{equation}
        \iprod{u'_i,x} \ge g^{(i-1)}(x) \ge g^{(i)}(x) \ge \iprod{u_i,x} \qquad \text{or} \qquad \iprod{u'_i,x} \le g^{(i-1)}(x) \le g^{(i)}(x) \le \iprod{u_i,x}.\label{eq:sandwich}
    \end{equation} This would imply
    \begin{equation}
        \Pr{g^{(i)}(x) - f(x) > \tau \wedge g^{(i-1)}(x) - f(x) \le \tau} \le \Pr{\iprod{u_i,x} - f(x) > \tau \wedge \iprod{u'_i,x} - f(x) \le \tau} \label{eq:hybrid_steps}
    \end{equation} because either the left-hand side of \eqref{eq:hybrid_steps} is zero, or or the event on the left-hand side immediately implies the one on the right-hand side.
    
    Denote by $\brc{(\iprod{w_i,\cdot},S_i)}_{i\in[m]}$ the pieces of some realization of $f$. We would then have
    \begin{align}
        \MoveEqLeft \Pr{g(x) - f(x) > \tau \wedge g'(x) - f(x) \le \tau} \\
        &\le \sum^m_{i=1}\Pr{\iprod{u_i,x} - f(x) > \tau \wedge \iprod{u'_i,x} - f(x) \le \tau} \\
        &= \sum^m_{\ell = 1}\sum^m_{i=1} \Pr{x \in S_{\ell} \wedge \iprod{u_i-w_{\ell},x} > \tau \wedge \iprod{u'_i-w_{\ell},x} \le \tau} \\
        &\le \sum^m_{\ell = 1}\sum^m_{i=1} \Pr{\iprod{u_i-w_{\ell},x} > \tau \wedge \iprod{u'_i-w_{\ell},x} \le \tau} \le O\left({\eta m^2}/{\tau}\right),
    \end{align}
    where the first step follows by triangle inequality and \eqref{eq:hybrid_steps}, and the last step follows by Lemma~\ref{lem:stability_ez}.

    To complete the proof, we now turn to proving that the quantities $g^{(i)}(x)$ and $g^{(i-1)}(x)$ are sandwiched between $\iprod{u'_i,x}$ and $\iprod{u_i,x}$, which will imply \eqref{eq:hybrid_steps}. Suppose that $g^{(i-1)}(x) = \iprod{u^{(i-1)}_j,x}$ for some index $j$.

    \noindent\textbf{Case 1}: $\iprod{u'_i,x} \ge \iprod{u^{(i-1)}_j,x}$.
    
    In this case $\min_{b\in \calI_a}\iprod{u^{(i-1)}_b,x} \le \iprod{u'_i,x}$ for all $a$. If $\iprod{u_i,x} \ge \iprod{u'_j,x}$, then changing $u'_i$ to $u_i$ will not change the values of any of the clauses. So suppose $\iprod{u_i,x} < \iprod{u'_j,x}$, in which case the value of the function cannot increase. Then if index $i$ appears in any clause $\calI_a$ for which $\min_{b\in \calI_a}\iprod{u^{(i-1)}_b,x} = \iprod{u^{(i-1)}_j,x}$, then $g^{(i)}(x) \ge \iprod{u_i,x}$. Otherwise, the value of the function stays the same. We conclude that the first inequality in \eqref{eq:sandwich} holds.
    
    \noindent\textbf{Case 2}: $\iprod{u'_i,x} < \iprod{u^{(i-1)}_j,x}$.
    
    In this case there is some $\calI_a$ for which $\iprod{u^{(i-1)}_j,x} = \min_{b\in\calI_a}\iprod{u^{(i-1)}_b,x}$ and in which index $i$ does not appear. If $\iprod{u_i,x} \le \iprod{u'_i,x}$, then changing $u'_i$ to $u_i$ will not change the value of this $\calI_a$ clause, and the values of the other clauses will not increase, so the value of the function will not change. So suppose $\iprod{u_i,x} > \iprod{u'_i,x}$. Changing $u'_i$ to $u_i$ will not affect any clause $\calI_a$ not containing $i$ or for which $\min_{b\in\calI_a}\iprod{u^{(i-1)}_b,x} \le u'_i$. For all other clauses, their value will either stay the same or increase to $u_i$, in which case $g^{(i)}(x) \le \iprod{u_i,x}$. We conclude that the second inequality in \eqref{eq:sandwich} holds.
\end{proof}

\subsection{Netting Over Piecewise Linear Functions}
\label{sec:nettingjuntas}

Suppose we have recovered an $\ell$-dimensional subspace $\td{W}$ that approximately lies within $V$. In this section we show how to produce a finite list of candidate kickers with relevant subspace $\td{W}$, one of which is guaranteed to approximate $F$ restricted to some $\ell$-dimensional subspace $W$. Ignoring the finiteness of this list for now, we first show that as long as $\td{W}$ is sufficiently close to lying within $V$, there exists \emph{some} kicker close to \emph{some} restriction $F|_W$.

\begin{lemma}\label{lem:approxsubspace}
    Let $\td{w}_1,\ldots,\td{w}_{\ell}$ be a frame $\nu$-nearly within $V$, with span $\td{W}$. There exist an $\ell$-dimensional subspace $W\subset V$ and a $\Lip$-Lipschitz kicker $\td{F}^*$ with relevant subspace $\td{W}$ which is $(M,2\sqrt{\nu}\cdot\ell \Lip)$-structurally-close to $F|_W$.
\end{lemma}

\begin{proof}[Proof of Lemma~\ref{lem:approxsubspace}]
    By Lemma~\ref{lem:frame}, there exist orthonormal vectors $w_1,\ldots,w_{\ell}$ for which $\norm{w_i - \td{w}_i} \le 2\sqrt{\nu \ell}$. Let $W$ be their span.
    
    The function $F|_W$ is a continuous piecewise-linear function with at most $M$ pieces, so by Theorem~\ref{thm:tropical} and Lemma~\ref{lem:normbound}, there exist vectors $u_1,\ldots,u_M\in W$ and subsets $\calI_1,\ldots,\calI_m\subseteq[M]$ for which $F(x) = \max_{j\in[m]}\min_{i\in \calI_j} \iprod{u_i,x}$ and $\norm{u_i}\le\Lip$ for all $i$. For any $i\in[M]$, write $u_i = \sum_{i'\in [\ell]} \alpha_{i,i'}w_{i'}$. Define $\td{u}^*_i \triangleq \sum_{i'\in[\ell]}\alpha_{i,i'}\td{w}_{i'}$ and define the kicker $\td{F}^*$ with relevant subspace $\td{W}$ by $\td{F}^*(x)\triangleq \max_{j\in[m]}\min_{i\in \calI_j} \iprod{\td{u}^*_i,x}$.
    
    Note that for any $i$, \begin{equation}
        \norm{\td{u}^*_i - u_i} = \sum_{i'\in[\ell]}\alpha_{i,i'}\norm{\td{w}_{i'} - w_{i'}} \le 2\sqrt{\nu \ell} \cdot \sum_{i'}\abs{\alpha_{i,i'}} 
        \le 2\sqrt{\nu}\cdot\ell\norm{u_i} \le 2\sqrt{\nu}\cdot\ell \Lip,
    \end{equation} where the penultimate step is by Cauchy-Schwarz, so $\td{F}^*$ is $(M,2\sqrt{\nu}\cdot\ell \Lip)$-structurally-close to $F|_W$ as claimed. Lastly, note that $\norm{\td{u}^*_i} = \norm{u_i} \le \Lip$, so $\td{F}^*$ is indeed $\Lip$-Lipschitz.
\end{proof}



We now show that the existential guarantee of Lemma~\ref{lem:useorder} implies that if we enumerate over a fine enough net of kickers, then we can recover an approximation to $\td{F}^*$ from Lemma~\ref{lem:approxsubspace} in time singly exponential in $\poly(M)$.

\begin{algorithm2e}
\DontPrintSemicolon
\caption{\textsc{EnumerateKickers}($\td{W}$, $\epsilon'$)}
\label{alg:enumerate}
    \KwIn{Subspace $\td{W}$ spanned by orthonormal vectors $\td{w}_1,\ldots,\td{w}_{\ell}$, granularity $\epsilon'>0$}
    \KwOut{List of kickers $\td{F}$ with relevant subspace $\td{W}$}
        $\calL\gets\emptyset$.\;
        Let $\calN$ be an $\epsilon'\Lip$-net over the set of vectors in $\td{W}$ with norm at most $\Lip$.\label{step:definenet}\;
        \For{$\td{u}_1,\ldots,\td{u}_M \in\calN$}{
            \For{functions $A:\Omega_M\to[M]$}{
                Let $\td{F}$ be the kicker given by \begin{equation}
                    \td{F}(x) = \sum_{\omega\in\Omega_M}\bone*{\brc{\iprod{\td{u}_i,x}}_{i\in[M]}\vdash \omega} \cdot \iprod{\td{u}_{A(\omega)},x}.
                \end{equation}
                Append $\td{F}$ to $\calL$.\;
            }
        }
        \Return $\calL$.\;
\end{algorithm2e}

\begin{lemma}\label{lem:existsnet}
    Take any $\epsilon'>0$. Given a frame $\td{w}_1,\ldots,\td{w}_{\ell}$ with span $\td{W}$, for any $\Lip$-Lipschitz kicker $\td{F}^*$ with relevant subspace $\td{W}$, there exists a kicker $\td{F}$ with relevant subspace $\td{W}$ in the output $\calL$ of {\sc EnumerateKickers}($\td{W},\epsilon'$) which is $(M,\epsilon'\Lip)$-structurally-close to $\td{F}$. Furthermore, $\abs{\calL} \le M^{M^2}\cdot (1 + 2/\epsilon')^{\ell}$.

    In particular, if $\td{w}_1,\ldots,\td{w}_{\ell}$ is a frame $\nu$-nearly within $V$, then for $\epsilon' = 2\sqrt{\nu}\cdot\ell$, $\calL$ contains a kicker $\td{F}$ which is $(M,\Cjunta\sqrt{\nu})$-structurally-close to $F|_W$ for some $\ell$-dimensional subspace $W\subseteq V$, where \begin{equation}
        \Cjunta\triangleq 4k\Lip. \label{eq:Cjunta}
    \end{equation} Furthermore, $\abs{\calL} \le M^{M^2} O({1}/{\sqrt{\nu}})^{\ell}$ in this case.
\end{lemma}

\begin{proof}
    By Lemma~\ref{lem:useorder}, the function $\td{F}^*$ in the hypothesis can be written in the form $\td{F}^*(x) = \sum_{\omega\in\Omega_M}\bone*{\brc{\iprod{\td{u}^*_i,x}}_{i\in[M]}\vdash \omega} \cdot \iprod{\td{u}^*_{A(\omega)},x}$ for some vectors $\brc{\td{u}^*_i}_{i\in[M]}$ and function $A:\Omega_M\to[M]$.

    Because $\calN$ in Step~\ref{step:definenet} of {\sc EnumerateKickers} is an $\epsilon'\Lip$-net over the set of vectors in $\td{W}$ with norm at most $\Lip$, there exist vectors $\td{u}_1,...,\td{u}_M\in\calN$ for which $\norm{\td{u}_i - \td{u}^*_i} \le \epsilon'\Lip$. If we define $\td{F}$ by $\td{F}(x) = \sum_{\omega\in\Omega_M}\bone*{\brc{\iprod{\td{u}_i,x}}_{i\in[M]}\vdash \omega} \cdot \iprod{\td{u}_{A(\omega)},x}$, then by design, $\td{F}$ is $(M,\epsilon'\Lip)$-structurally-close to $\td{F}$.

    It remains to bound the size of $\calL$. For any $\epsilon'>0$ there is an $\epsilon'$-net $\calN'_{\epsilon'}$ for the $L_2$ unit ball in $\td{W}$ of size at most $(1 + 2/\epsilon')^{\ell}$.
    Define $\calN \triangleq \Lip\cdot \calN'_{\epsilon'}$.
    Furthermore, there are $|\Omega_M|^M \le M^{M^2}$ functions $A: \Omega_M\to[M]$. This yields the desired bound on $\abs{\calL}$.

    The final part of the lemma follows by invoking Lemma~\ref{lem:approxsubspace} and noting that the lattice polynomial representation of $\td{F}^*$ and that of $F|_W$ are identical in the proof of Lemma~\ref{lem:approxsubspace}, so the structural closeness of $\td{F}$ to $F|_W$ follows by triangle inequality.
\end{proof}

\subsection{Netting Over Neural Networks}
\label{sec:nettingnets}

Enumerating over arbitrary kickers with $M$ pieces requires runtime scaling exponentially in $\poly(M)$. For ReLU networks of size $S$, $M$ could be as large as $\exp(S)$, so naively using {\sc EnumerateKickers} in our application to learning ReLU networks would incur doubly exponential dependence on $k$ in the runtime. In this section we show how to enumerate over ReLU networks more efficiently. We first prove the analogue of Lemma~\ref{lem:approxsubspace} for ReLU networks.

\begin{lemma}\label{lem:approxsubspace_nns}
    Suppose $F$ additionally satisfies Assumption~\ref{assume:bounds2}. Let $\td{w}_1,\ldots,\td{w}_{\ell}$ be a frame $\nu$-nearly within $V$, with span $\td{W}$. There exist an $\ell$-dimensional subspace $W\subset V$ and weight matrix $\vW^*_0\in\R^{k_0\times d}$ with rows in $\td{W}$ for which \begin{equation}
        \norm{\vW_0\Pi_W - \vW^*_0}_{\op} \le 2\sqrt{\nu}\cdot \ell\sqrt{k}\cdot\normbound\label{eq:werror}
    \end{equation} and for which $\norm{\vW^*_0}_{\op} \le \normbound$.
\end{lemma}

\begin{proof}
    As in the proof of Lemma~\ref{lem:approxsubspace}, Lemma~\ref{lem:frame} yields orthonormal vectors $w_1,\ldots,w_{\ell}$ for which $\norm{w_i - \td{w}_i} \le 2\sqrt{\nu\ell}$. Let $W$ be their span.

    If $F$ has weight matrices $\vW_0\in\R^{k_0\times d},\vW_1\in\R^{k_1\times k_0},\ldots,\vW_{L+1}\in\R^{1\times k_{L}}$, then $F|_W$ is a ReLU network with weight matrices $\vW_0\Pi_W, \vW_1,\ldots,\vW_{L+1}$. Denoting the rows of $\vW_0\Pi_W\in\R^{k_0\times d}$ as $u_1,\ldots,u_{k_0}$, we may write them as $u_i = \sum_{i'\in[\ell]}\alpha_{i,i'}w_{i'}$ for $i\in[k_0]$.

    Define $\td{u}^*_i \triangleq \sum_{i'\in[\ell]}\alpha_{i,i'}\td{w}_{i'}$. As in the proof of Lemma~\ref{lem:approxsubspace}, we have that \begin{equation}
        \norm{\td{u}^*_i - u_i} \le 2\sqrt{\nu}\cdot \ell\norm{u_i} \le 2\sqrt{\nu}\cdot \ell\normbound,
    \end{equation} where in the last step we have used the fact that the maximum norm of any row of $\vW_0\Pi_W$ is at most the maximum norm of any row of $\vW_0$, which is upper bounded by $\norm{\vW_0}_{\op} \le \normbound$.

    Let $\td{\vW}^*_0$ denote the matrix whose rows consist of $\td{u}^*_1,\ldots,\td{u}^*_{k_0}$. We have that \begin{equation}
        \norm{\vW_0\Pi_W - \td{\vW}^*_0}_{\op} \le \norm{\vW_0\Pi_W - \td{\vW}^*_0}_{F} \le 2\sqrt{\nu} \cdot \ell\sqrt{k} \cdot \normbound
    \end{equation} as claimed. Finally, the bound on $\norm{\vW^*_0}_{\op}$ follows from the fact that $\vW^*_0 = \vW_0\cdot \vec{O}\cdot \Pi_W$ for an orthogonal matrix $\vec{O}$ mapping the frame $\brc{w_1,...,w_{\ell}}$ to $\brc{\td{w}_1,...,\td{w}_{\ell}}$.
\end{proof}



\begin{algorithm2e}
\DontPrintSemicolon
\caption{\textsc{EnumerateNetworks}($\td{W}$, $\epsilon'$)}
\label{alg:enumerate_nns}
    \KwIn{Subspace $\td{W}$ spanned by orthonormal vectors $\td{w}_1,\ldots,\td{w}_{\ell}$, granularity $\epsilon'>0$}
    \KwOut{List of size-$S$ ReLU networks $\td{F}$ with relevant subspace $\td{W}$}
        $\calL\gets\emptyset$.\;
        \For{tuples $(\td{k}_0,...,\td{k}_{L+1})\in\Z^{L+2}_{> 0}$ satisfying $\sum^{L+1}_{i=0} \td{k}_i = S$}{
            For every $0\le i\le L+1$, let $\calN_i$ be an $\epsilon'$-net (in operator norm) over the set of matrices in $\R^{\td{k}_i\times \td{k}_{i-1}}$ with operator norm at most $\normbound + \epsilon'$. \;
            \For{$\td{\vW}_0\in\calN_0,...,\td{\vW}_{L+1}\in\calN_{L+1}$}{
                Define the ReLU network $\td{F}$ with weight matrices $\clip_{\epsilon'}(\vW_0),...,\clip_{\epsilon'}(\vW_{L+1})$.\;
                Append $\td{F}$ to $\calL$.\;
            }
        }
        \Return $\calL$.\;
\end{algorithm2e}

We can now show the analogue of Lemma~\ref{lem:existsnet} for ReLU networks.

\begin{lemma}\label{lem:existsnet_nns}
    Take any $0<\epsilon'\le \normbound$ and any frame $\td{w}_1,\ldots,\td{w}_{\ell}$ with span $\td{W}$. For any ReLU network $\td{F}^*$ of size $S$ with relevant subspace $\td{W}$ and depth $L+2$ whose weight matrices have operator norm at most $\normbound$, there exists a ReLU network $\td{F}$ with relevant subspace $\td{W}$ in the output $\calL$ of {\sc EnumerateNetworks}($\td{W},\epsilon'$) which is $(2^S,2^{O(L)}\normbound^{L+1} \epsilon')$-structurally-close (as a piecewise-linear function) to $\td{F}$. Furthermore, $\abs{\calL} \le 2^{O(S)}\cdot (1 + 4\normbound/\epsilon')^{O(S^2)}$.

    In particular, if $\td{w}_1,...,\td{w}_{\ell}$ is a frame $\nu$-nearly within $V$, then for $\epsilon' = 2\sqrt{\nu}\cdot \ell\sqrt{k}\cdot \normbound$, $\calL$ contains a ReLU network $\td{F}$ which is $(M,\Cnet\sqrt{\nu})$-structurally-close to $F|_W$ for some $\ell$-dimensional subspace $W\subseteq V$, where \begin{equation}
        \Cnet \triangleq 2^{O(L)}\normbound^{L+2}k^{3/2} \label{eq:Cnet}
    \end{equation} Furthermore, $\abs{\calL} \le O(1/\sqrt{\nu})^{O(S^2)}$ in this case.
\end{lemma}

\begin{proof}
    Let $\vW'_0\in\R^{k'_0\times d},\ldots,\vW_{L+1}\in\R^{1\times k_{L}}$ denote the weight matrices of $\td{F}^*$. Consider the iteration of the outer loop of {\sc EnumerateNetworks} in which the architecture of $\td{F}^*$ is guessed correctly, that is, for which $\td{k}_i = k'_i$ for all $0\le i\le L + 1$.
    By the choice of nets, there is some iteration of the inner loop of the algorithm for which the weight matrices $\brc{\td{\vW}_i}$ satisfy \begin{equation}
        \norm{\vW'_i - \td{\vW}_i}_{\op} \le \epsilon' \ \forall \ 0\le i \le L + 1. \label{eq:otheris}
    \end{equation}
    Define the ReLU network $\td{F}$ with relevant subspace $\td{W}$ to have weight matrices $\td{\vW}_0,\td{\vW}_1,\ldots,\td{\vW}_{L+1}$. By the fact that operator norm closeness implies entrywise closeness, together with Fact~\ref{fact:clipping} and Theorem~\ref{thm:tropical_nns}, there are lattice polynomial representations for $\td{F}^*$ and $\td{F}$ with identical clauses, and for which the vectors at the leaves consist of $\vW'_{L+1}\Sig_{L}\vW'_{L}\cdots \Sig_0\vW'_0\Pi_W$ and $\td{\vW}_{L+1}\Sig_{L}\td{\vW}_{L}\cdots \Sig_0\td{\vW}_0$ respectively for all possible diagonal matrices $\Sig_i\in\brc{0,1}^{k'_i\times k'_i}$. For any such choice of matrices $\brc{\Sig_i}$, note that \begin{align}
        \MoveEqLeft \norm{\vW'_{L+1}\Sig_{L}\vW'_{L}\cdots \vW'_0 - \td{\vW}_{L+1}\Sig_{L}\td{\vW}_{L}\cdots \td{\vW}_0} \\
        &\le \norm{(\vW'_{L+1} - \td{\vW}_{L+1})\Sig_{L}\vW'_{L}\cdots \vW'_0} + \cdots + \norm{\td{\vW}_{L+1}\Sig_{L}\td{\vW}_{L}\cdots(\vW'_0 - \td{\vW}_0)} \\
        &\le \norm{\vW'_{L+1} - \td{\vW}_{L+1}} \prod^{L}_{i=0}\norm{\vW'_i}_{\op} + \cdots + \prod^{L+1}_{i=1}\norm{\td{\vW}_i}_{\op}\norm{\vW'_0 - \td{\vW}_0}_{\op} \\
        &\le (L + 2)\cdot (\normbound + \epsilon')^{L + 1}\cdot \epsilon' \\
        &\le 2^{O(L)}\normbound^{L+1}\cdot \epsilon',\label{eq:genericbound}
    \end{align} where in the last step we used the assumption that $\epsilon'\le \normbound$. This implies the claim about structural closeness.

    We next bound the size of $\abs{\calL}$. For any choice of $\td{k}_0,...,\td{k}_{L+1}$, note that by Corollary~\ref{cor:opnet}, \begin{align}
        \abs*{\calN_{\td{k}_0}\times\cdots \times \calN_{\td{k}_{L+1}}} &\le (1 + 4\normbound/\epsilon')^{L\td{k}_0 + \td{k}_0\td{k}_1 + \cdots + \td{k}_{L}\td{k}_{L+1} + \td{k}_{L+1}} \\
        &\le (1 + 4\normbound/\epsilon')^{O(S^2)} 
    \end{align}
    where in the penultimate step we used that \begin{equation}
        L\td{k}_0 + \td{k}_0\td{k}_1 + \cdots + \td{k}_{L}\td{k}_{L+1} + \td{k}_{L+1} \le (L + \td{k}_0 + \cdots + \td{k}_{L+1})(\td{k}_0 + \cdots + \td{k}_{L+1} + 1) = (L + S)(S + 1) \le O(S^2).
    \end{equation}
    There are $\binom{S + L + 1}{L + 1} = 2^{O(S)}$ choices of $(\td{k}_0,\ldots,\td{k}_{L + 1})$ in the outer loop of {\sc EnumerateNetworks}, so $\abs{\calL} \le 2^{O(S)}\cdot (1 + 4\normbound/\epsilon')^{O(S^2)}$ as claimed.

    Finally, to obtain the last part of the lemma, we can take $\td{F}^*$ above to have the same weight matrices as $F$ except for the input layer, which we will take to be $\vW'_0 \triangleq \td{\vW}^*_0$ for the weight matrix guaranteed by Lemma~\ref{lem:approxsubspace_nns}. By \eqref{eq:werror}, this choice of $\vW'_0$ is close to $\vW_0\Pi_W$ for some subspace $W\subseteq V$. Take $\epsilon' = 2\sqrt{\nu}\cdot \ell\sqrt{k}\cdot \normbound$. For $\brc{\td{\vW}_i}$ satisfying \eqref{eq:otheris}, by triangle inequality \eqref{eq:werror} we get that \begin{equation}
        \norm{\vW_0\Pi_W - \td{\vW}_0}_{\op} \le \norm{\vW_0\Pi_W - \vW'_0}_{\op} + \norm{\vW'_0 - \td{\vW}_0}_{\op} \le 2\epsilon'.\label{eq:izero}
    \end{equation}
    Using this, by a calculation analogous to the one leading to \eqref{eq:genericbound}, we find that $\td{F}$ is $(2^S,2^{O(L)}\normbound^{L+1}\epsilon')$-structurally-close to $F|_W$, from which the claim follows by our choice of $\epsilon' = 2\sqrt{\nu}\cdot \ell\sqrt{k}\cdot\normbound$. In this case, we get that $\abs{\calL} \le 2^{O(S)} (1 + 2/\sqrt{\nu})^{O(S^2)} \le O(1/\sqrt{\nu})^{O(S^2)}$ as claimed.
\end{proof}

With subroutines for enumerating over ReLU networks and kickers in hand, we can now formally state our algorithm, {\sc FilteredPCA} (see Algorithm~\ref{alg:filteredpca} below). The algorithm as stated applies to the case where $F$ is a neural network satisfying Assumptions~\ref{assume:bounds} and \ref{assume:bounds2}, but we can easily modify the algorithm to work in the case where $F$ is only a kicker satisfying Assumption~\ref{assume:bounds} by replacing the call to \textsc{EnumerateNetworks}($\td{W},2\sqrt{\nu_0}\cdot \ell\sqrt{k}\cdot \normbound$) in Line~\ref{step:enumerate1} with a call to \textsc{EnumerateKickers}($\td{W},2\sqrt{\nu_0}\cdot\ell$), the call to {\sc EnumerateNetworks}($\td{W},\normbound^{-L-1}2^{-\Omega(L)}\cdot \epsilon/\sqrt{k}$) in Line~\ref{step:enumerate2} with a call to {\sc EnumerateKickers}($\td{W},\epsilon/(2\sqrt{k}\Lip)$), and the assignment $N'\gets \poly(\normbound^{L+2},k,1/\epsilon)\cdot\log(1/\delta)$ in Line~\ref{step:assign} with the assignment $N'\gets \poly(\Lip,k,1/\epsilon)\cdot\log(1/\delta))$.

\begin{algorithm2e}\caption{{\sc FilteredPCA}($\calD,\epsilon,\delta$)}\label{alg:filteredpca}
    \DontPrintSemicolon
    \KwIn{Sample access to $\calD$, target error $\epsilon$, failure probability $\delta$}
    \KwOut{Size-$S$ ReLU network $\td{F}: \R^d\to\R$ for which $\norm{\td{F} - F} \le O(\epsilon)$ with probability at least $1 - \delta$}
    $\calW \gets \emptyset$.\;
    $\tau \gets c\sqrt{k}\cdot \Lip$ as in \eqref{eq:taudef}.\;
    $\nu_0\gets \poly(k^k,1/\ulam^k,M^k,\Lip)^{-1}$, where $\ulam$ is defined in \eqref{eq:lambound}.\;
    $\xi\gets O\left(k \left(\sqrt{\nu_0 k}\cdot M^2/c\right)^{1 - 1/k}\right)$ as in \eqref{eq:xidef}.\;
    $N\gets \Omega(\brc{\Max{d}{\log(2k/\delta)}}/\xi^2)$.\;
    \For{$0 \leq \ell \leq k-1$}{
        Draw samples $(x_1,y_1),\ldots,(x_N,y_N)\sim\calD$.\label{step:drawsamples}\;
        If $\calW = \brc{\td{w}_1,\ldots,\td{w}_\ell}$, let $\td{W}$ denote the span of these vectors.\;
        $\calL \gets${\sc EnumerateNetworks}($\td{W},2\sqrt{\nu_0}\cdot \ell\sqrt{k}\cdot \normbound$).\;\label{step:enumerate1}
        \For{$\td{F}\in\calL$}{
            Form the matrix
            \begin{equation}
                \td{\vM}^{\td{W}}_{\emp} \triangleq \Pi_{\td{W}^{\perp}}\left(\sum^N_{i=1}\bone*{\abs{y_i - \td{F}(\Pi_{\td{W}}x)} > \tau}\cdot (x_ix_i^{\top} - \Id)\right)\Pi_{\td{W}^{\perp}}.\label{eq:tdMl}
            \end{equation}
            Run {\sc ApproxBlockSVD}($\td{\vM}^{\td{W}}_{\emp},\ulam/1000,\delta/(2\abs{\calL}k)$) to obtain approximate top singular vector $\td{w}^{\ell+1}$.\;
            $\lambda\gets (\td{w}^{\ell+1})^{\top}\td{\vM}^{\td{W}}_{\emp}\td{w}^{\ell+1}$.\label{step:computelambda}\;
            \If{$\lambda \ge 9\ulam/16$}{
                Append $\td{w}^{\ell+1}$ to $\calW$ and exit out of this inner loop and increment $\ell$. \label{step:appendw}\;\label{step:foundit}
            }
        }
        \If{no $\td{w}^{\ell+1}$ was appended to $\calW$}{
            \Return $\calW$.\;
        }
    }
    Let $\td{W}$ denote the span of the vectors in $\calW$. \label{step:finalW}\;
    $\calL\gets${\sc EnumerateNetworks}($\td{W},\normbound^{-L-1}2^{-\Omega(L)}\cdot \epsilon/\sqrt{k}$).\;\label{step:enumerate2}
    $N'\gets \poly(\normbound^{L+2},k,1/\epsilon)\cdot\log(1/\delta)$.\;\label{step:assign}
    \For{$\td{F}\in\calL$}{
        Form an empirical estimate $\hat{\epsilon}$ for $\norm{\td{F} - F}$ by drawing $N'$ samples.\;
        \If{$\hat{\epsilon}\le 3\epsilon$}{
            \Return $\td{F}$.\;
        }
    }
\end{algorithm2e}

\subsection{Perturbation Bounds}
\label{sec:perturb}

We now show how to leverage Lemma~\ref{lem:stability} to show that even with access to a subspace $\wt{W}$ which is only approximately within $V$ as well as the restriction of $F$ to that subspace, we can recover another vector orthogonal to $\wt{W}$ which mostly lies within $V$.



The first step is to show that in this approximate setting, the analogue of $\vM^W$ from Section~\ref{sec:idealized} is spectrally close to $\vM^W$. It is in showing this perturbation bound that we invoke the stability result of Section~\ref{sec:stability}.

\begin{lemma}\label{lem:spectral_close}
    Suppose $F$ only satisfies Assumption~\ref{assume:bounds} (resp. both Assumptions~\ref{assume:bounds} and \ref{assume:bounds2}). Let $\td{w}_1,\ldots,\td{w}_{\ell}\in\S^{d-1}$ be a frame $\nu$-nearly within $V$, with span $\td{W}$. For $*\in\brc{\mathsf{piecewise},\mathsf{network}}$, define
    \begin{equation}
        \xi_*(\nu) \triangleq O\left(\Max{k\left(\frac{\Constant\sqrt{\nu}M^2}{c\sqrt{k}\Lip}\right)^{1-1/k}}{\sqrt{\nu k}}\right) \label{eq:xidef}
    \end{equation}
    and suppose $N \ge \Omega(\brc{\Max{d}{\log(1/\delta)}}/\xi_*^2)$.

    Given subspace $W\subseteq V$ and $\td{F}$ for which $F|_W$ and $\td{F}$ are $(M,\Cjunta\sqrt{\nu})$-structurally-close (resp. $(M,\Cnet\sqrt{\nu})$-structurally close),
    then we have that \begin{equation}
        \norm{\td{\vM}^{\td{W}}_{\emp} - \vM^W}_{\op} \le 3\xi(\nu)
    \end{equation} with probability at least $1 - \delta$.
\end{lemma}

\begin{proof}
    For convenience denote $\td{\vM}^{\td{W}}_{\emp}$ and $\vM^W$ by $\td{\vM}_{\emp}$ and $\vM$ respectively. Also, depending on whether $F$ only satisfies Assumption~\ref{assume:bounds} or both Assumptions~\ref{assume:bounds} and \ref{assume:bounds2}, define $\Constant\triangleq \Cjunta$ or $\Constant\triangleq \Cnet$ respectively. It will also be convenient to define \begin{equation}
        \vM' \triangleq \Pi_{W^{\perp}}\E[x,y]*{\bone{\abs{y - F|_W(x)} > \tau}\cdot (xx^{\top} - \Id)}\Pi_{W^{\perp}}
    \end{equation} as well as the population version of $\td{\vM}_{\emp}$, that is, $\td{\vM}\triangleq \E[(x_1,y_1),\ldots,(x_N,y_N)]{\td{\vM}_{\emp}}$.

    We will upper bound \begin{equation}
        \norm{\td{\vM}_{\emp} - \vM}_{\op} \le \norm{\td{\vM}_{\emp} - \td{\vM}}_{\op} + \norm{\td{\vM} - \vM'}_{\op} + \norm{\vM' - \vM}_{\op}.
    \end{equation}
    by upper bounding each of the summands on the right-hand side by $\xi_*$.

    By Lemma~\ref{lem:matrixconcentration} and our choice of $N$, $\norm{\td{\vM}_{\emp} - \td{\vM}}_{\op}\le \xi_*$ with probability at least $1 - \delta$.

    To upper bound $\norm{\td{\vM} - \vM'}_{\op}$, we can naively upper bound 
    \begin{equation}
        \norm*{\E[x,y]*{\bone{\abs{y - F|_W(x)} > \tau}\cdot (xx^{\top} - \Id)}} \le 2,
    \end{equation}
    so by Lemma~\ref{lem:projecterror} and Lemma~\ref{lem:frame} we have 
    \begin{equation}
        \norm{\td{\vM} - \vM'}_{\op} \le 2\sqrt{2} \cdot d_C(\td{W},W) \le 4\sqrt{\nu\cdot k} \le \xi_*
    \end{equation}

    Finally, we upper bound $\norm{\vM' - \vM}_{\op}$. For any test vector $v\in\S^{d-1}$ orthogonal to $W$,
    \begin{align}
        v^{\top}(\vM - \vM')v &= \E[x]*{\left(\bone{\abs{y - F|_W(x)} > \tau} - \bone{\abs{y - \td{F}(\Pi_{\td{W}} x)} > \tau}\right)\cdot (\iprod{v,x}^2 - 1)} \\
        &\le \Pr[x]*{\sgn(\abs{y - F|_W(x)} - \tau) \neq \sgn(\abs{y - \td{F}(\Pi_{\td{W}} x)} - \tau)}^{1 - 1/k}\cdot O(k) \label{eq:holder} \\
        &\le O\left(k\left(\frac{\Constant \sqrt{\nu} M^2}{\tau}\right)^{1 - 1/k}\right) = O\left(k\left(\frac{\Constant\sqrt{\nu}M^2}{c\sqrt{k}\Lip}\right)^{1-1/k}\right)\le \xi_*\\
    \end{align}
    where the second step follows by Holder's and the fact that $\E[g\sim\N(0,1)]{(g^2 - 1)^k}^{1/k} \le O(k)$, and the third step follows by Lemma~\ref{lem:stability}, which we may apply because $\td{F}$ and $F|_W$ are $(M,4\sqrt{\nu}\cdot \ell \Lip)$-structurally-close.
\end{proof}

Finally, we use the above perturbation bound to show that in a single iteration of the main outer loop of {\sc FilteredPCA}, if there is some variance unexplained by the subspace $\wt{W}$ found so far (see \eqref{eq:leftover_var}), then we will find another ``good'' direction orthogonal to $\td{W}$ which is also approximately within the span of $V$. Note that this claim has two components: \emph{completeness}, i.e. in the list of candidate functions we have enumerated, there is \emph{some} function for which the top singular vector of \eqref{eq:tdMl} is a good direction, and \emph{soundness}, i.e. whatever direction is ultimately chosen in Step~\ref{step:foundit} of {\sc FilteredPCA} is a good direction.

\begin{lemma}\label{lem:main_stable}
    Suppose $F$ only satisfies Assumption~\ref{assume:bounds} (resp. both Assumptions~\ref{assume:bounds} and \ref{assume:bounds2}). Suppose $\nu \le \epsilon^2/(4k\Cjunta^2)$ (resp. $\nu \le \epsilon^2/(4k\Cnet^2)$).
    For $0\le \ell < k$, let $\td{w}_1,\ldots,\td{w}_{\ell}$ be a frame $\nu$-nearly within $V$, with span $\td{W}$. Define $\xi = \xi_{\mathsf{piecewise}}(\nu)$ (resp. $\xi = \xi_{\mathsf{network}}(\nu)$) according to \eqref{eq:xidef}, and suppose $N \ge \Omega(\brc{\Max{d}{\log(1/\delta)}}/\xi^2)$ and $\tau = c\sqrt{k}\cdot \Lip$.

    Suppose $\xi\le \ulam/6$, and suppose
    \begin{equation}
        \E[x\sim\N(0,\Id)]{(F(x) - F(\Pi_{\td{W}}x))^2} \ge \epsilon^2.\label{eq:leftover_var}
    \end{equation}
    Let $\calL$ be the output of {\sc EnumerateKickers}($\td{W},2\sqrt{\nu}\cdot \ell$) (resp. {\sc EnumerateNetworks}($\td{W},2\sqrt{\nu}\cdot \ell\sqrt{k}\cdot\normbound$)). With probability at least $1 - |\calL|\cdot \delta$ over the randomness of the $N$ samples, the following hold:
    \begin{enumerate}
        \item \textbf{Completeness}: There exists some $\td{F}\in\calL$ such that, if $\td{\vM}^{\td{W}}_{\emp}$ is defined according to \eqref{eq:tdMl}, its top singular value is at least $\ulam - 3\xi$.
        \item \textbf{Soundness}: For any $\td{F}\in\calL$ for which $\norm{\td{\vM}^{\td{W}}_{\emp}}_{\op} \ge \ulam - 3\xi$, the top singular vector $w$ satisfies $\norm{\Pi_V w} \ge 1 - c'\xi^2/\ulam^2$ for some absolute constant $c'>0$ and is orthogonal to $\td{W}$.
    \end{enumerate}
\end{lemma}

\begin{proof}    
    When the choice of $\td{F}$ is clear from context, for convenience we will denote $\vM^{W}$ and $\td{\vM}^{\td{W}}_{\emp}$ by $\vM$ and $\td{\vM}_{\emp}$ respectively.

    By Lemma~\ref{lem:existsnet} (resp. Lemma~\ref{lem:existsnet_nns}) and our assumed bound on $\nu$, there exists $\td{F}$ in the output of {\sc EnumerateKickers} (resp. {\sc EnumerateNetworks}) which is $(M,\epsilon/2k)$-structurally-close to $F|_W$ for some $\ell$-dimensional subspace $W\subsetneq V$.
    
    By triangle inequality, Lemma~\ref{lem:hybrid_L2}, and \eqref{eq:leftover_var}, and our assumed bounds on $\nu$, we have that $\norm{F - F|_W} \ge \epsilon/2$. So by Lemma~\ref{lem:idealized_find} and \eqref{eq:lambound}, we know $\norm{\vM} \ge \ulam$. 

    Because this $\td{F}$ is $(M,\Cjunta\sqrt{\nu})$-structurally-close (resp. $(M,\Cnet\sqrt{\nu})$-structurally close) to $F|_W$, Lemma~\ref{lem:spectral_close} implies that with probability $1 - \delta$, $\norm{\vM - \td{\vM}_{\emp}}_{\op} \le 3\xi$, so $\td{\vM}_{\emp}$ has top singular value at least $\ulam - 3\xi$. This proves completeness.

    Now take any $\td{F}$ for which $\norm{\td{\vM}_{\emp}}_{\op} \ge \ulam - 3\xi$. The fact that the top singular vector $w$ is orthogonal to $\td{W}$ is immediate. And by Lemma~\ref{lem:spectral_close}, with probability $1 - \delta$ over the samples, $\norm{\vM - \td{\vM}_{\emp}}_{\op} \le 3\xi$. So if we take $\lambda,\epsilon,\vec{A},\wh{\vec{A}}$ in Corollary~\ref{cor:topsing} to be $\ulam$, $3\xi$, $\vM$, and $\td{\vM}_{\emp}$ respectively, then because $\xi \le \ulam/6$, we get that the top singular vector $w$ of $\td{\vM}_{\emp}$ satisfies $\norm{\Pi_V w} \ge 1 - O(\xi^2/\ulam^2)$. This proves soundness, upon union bounding over all $\td{F}\in\calL$.
\end{proof}

\subsection{Putting Everything Together}
\label{sec:putting}

To conclude the proof of Theorems~\ref{thm:main_piecewise} and \ref{thm:main_nets}, we first show that for the subspace $\td{W}$ formed in Step~\ref{step:finalW}, if $\td{W}$ is sufficiently close to the true relevant subspace $V$ or if \eqref{eq:leftover_var} is violated, then one can run {\sc EnumerateKickers} (resp. {\sc EnumerateNetworks}) one more time to produce a function with small squared error relative to $F$.

\begin{lemma}\label{lem:finalgrid}
    Suppose $F$ only satisfies Assumption~\ref{assume:bounds} (resp. both Assumptions~\ref{assume:bounds} and \ref{assume:bounds2}). Define \begin{equation}
        \epsilon^* \triangleq \epsilon/(2\sqrt{k}\Lip) \qquad \text{(resp.} \ \epsilon^*\triangleq \normbound^{-L-1}2^{-\Omega(L)}\cdot \epsilon/\sqrt{k} \text{)} \label{eq:epsstar}
    \end{equation}

    Let $\td{w}_1,...,\td{w}_{\ell}$ be a frame with span $\td{W}$. If either 1) $\ell = k$ and this frame is $\epsilon^2/4k\Cjunta^2$-nearly (resp. $\epsilon^2/4k\Cnet^2$-nearly) within $V$, or 2) inequality \eqref{eq:leftover_var} is violated. Then the output $\calL$ of {\sc EnumerateKickers}($\td{W},\epsilon^*$) (resp. {\sc EnumerateNetworks}($\td{W},\epsilon^*$)) contains a function $\td{F}$ for which $\norm{F - \td{F}} \le O(\epsilon)$. Furthermore, $\abs{\calL} \le M^{M^2}\cdot O(\Lip/\epsilon)^k$ (resp. $\abs{\calL} \le O(\normbound^{L+2}2^{O(L)}/\epsilon)^{O(S^2)}$).

    In particular, if 1) or 2) holds for the subspace $\td{W}$ at the end of running {\sc FilteredPCA}, then the output $\td{F}$ of {\sc FilteredPCA} satisfies $\norm{F - \td{F}} \le O(\epsilon)$.
\end{lemma}

\begin{proof}
    We first show that if either 1) or 2) holds, then there exists $\td{F}$ in $\calL$ for which $\norm{\td{F} - F} \le O(\epsilon)$.

    Suppose 1) holds. If $F$ only satisfies Assumption~\ref{assume:bounds} (resp. Assumptions~\ref{assume:bounds} and \ref{assume:bounds2}), then by the final part of Lemma~\ref{lem:existsnet} (resp. Lemma~\ref{lem:existsnet_nns}), there is a function $\td{F}$ in $\calL$ which is $(M,\epsilon/2k)$-structurally-close (resp. $(2^S,\epsilon/2k)$-structurally-close) to $F|_W$ for $\ell$-dimensional subspace $W\subseteq V$. Because $\ell = k$ when 1) holds, this subspace must be $V$, so in fact $F|_W = F$ and therefore $\td{F}$ is structurally-close to $F$. By Lemma~\ref{lem:hybrid_L2}, we conclude that $\norm{\td{F} - F} \le \epsilon$.

    Suppose 2) holds. If $F$ only satisfies Assumption~\ref{assume:bounds} (resp. Assumptions~\ref{assume:bounds} and \ref{assume:bounds2}), then we can take $\td{F}^*$ in the first part of Lemma~\ref{lem:existsnet} (resp. Lemma~\ref{lem:existsnet_nns}) to be the function $x\mapsto F(\Pi_{\td{W}}x)$, which is clearly also a $\Lip$-Lipschitz kicker (resp. ReLU network of size $S$ whose weight matrices have operator norm at most $\normbound$) with relevant subspace $\td{W}$. It follows that $\calL$ contains some function $\td{F}$ which is $(M,\epsilon/(2\sqrt{k}))$- (resp. $(2^S,\epsilon/(2\sqrt{k}))$-structurally-close to $\td{F}^*$. By Lemma~\ref{lem:hybrid_L2}, we conclude that $\norm{\td{F} - F} \le 3\epsilon/2$.

    For the last part of the lemma, note that by Lemma~\ref{lem:estimate_var} in Appendix~\ref{app:estimate_var} that for any function $\td{F}$ for which $\norm{\td{F} - F}^2 \le \mu$, we can estimate $\norm{\td{F} - F}^2$ to error $O(\epsilon^2)$ from $O((\mu + \Lip^2 k)\log(1/\delta)/\epsilon^4)$ samples (resp. $O((\mu + \normbound^{2L+4}k)\log(1/\delta)/\epsilon^2)$). Note that for any $\td{F}\in\calL$, by the second part of Lemma~\ref{lem:hybrid_L2} we have that $\norm{\td{F} - F} \le O(\Lip\sqrt{k})$ (resp. $\norm{\td{F} - F} \le O(\normbound^{L+2}\sqrt{k})$).
\end{proof}

We can now conclude the proof of correctness for {\sc FilteredPCA}.

\begin{proof}[Proof of Theorem~\ref{thm:main_piecewise}]
    First note that the only randomness in {\sc FilteredPCA} comes from calling {\sc ApproxBlockSVD} and drawing samples, so henceforth we will condition on the event that the former always succeeds and on the success of Lemma~\ref{lem:matrixconcentration} for every batch of samples drawn in Step~\ref{step:drawsamples} of {\sc FilteredPCA}. By our choice of parameters in {\sc FilteredPCA} and a union bound, this event happens with probability at least $1 - \delta$.

    If $F$ satisfies Assumption~\ref{assume:bounds} only (resp. both Assumptions~\ref{assume:bounds} and \ref{assume:bounds2}), let $\xi(\nu) = \xi_{\mathsf{piecewise}}(\nu)$ and $C_* = \Cjunta$ (resp. $\xi(\nu) = \xi_{\mathsf{network}}(\nu)$ and $C_* = \Cnet$), recalling the definition from \eqref{eq:xidef}.

    Call $\nu \ge 0$ \emph{admissible} if $\nu \le \epsilon^2/4kC_*^2$ and $\xi(\nu) \le \ulam/6$. Let $\iota:\R\to\R$ be the function given by $\iota(\nu) = c'\xi(\nu)^2/\ulam^2$, where $c'$ is the absolute constant in Lemma~\ref{lem:main_stable}. Note that if we define \begin{equation}
        \beta \triangleq (c'/\ulam^2)\cdot O\left(k^2\cdot \left(\frac{C_*^2 M^4}{c^2 k\Lip^2}\right)^{1-1/k}\right),\label{eq:betadef}
    \end{equation}
    then $\iota(\nu) = \Max{(\beta\cdot \nu^{1 - 1/k})}{(k\nu)}$.

    Because we are conditioning on every invocation of {\sc ApproxBlockSVD} succeeding, the quantity $\lambda$ computed in Step~\ref{step:computelambda} is certainly $\xi(\nu)/2$-close to the true top singular value of $\td{\vM}^{\td{W}}$. So Lemma~\ref{lem:main_stable} tells us that in any iteration $\ell$ of the main loop in {\sc FilteredPCA}, if $\brc{\td{w}_1,\ldots,\td{w}_{\ell}}$ is a frame $\nu$-nearly within $V$ for admissible $\nu$, then either 1) we reach Line~\ref{step:appendw} in the inner loop and append some $\td{w}_{\ell+1}$ for which $\brc{\td{w}_1,\ldots,\td{w}_{\ell+1}}$ is a frame $\iota(\nu)$-nearly within $V$, or 2) \eqref{eq:leftover_var} is violated, in which case condition 2) of Lemma~\ref{lem:finalgrid} implies that {\sc FilteredPCA} would output a function $\td{F}$ for which $\norm{F - \td{F}} \le O(\epsilon)$.

    So all we need to verify is that there is a choice of $\nu_0$ for which the $k$ numbers
    \begin{equation}
        \nu_0,\iota(\nu_0),\ldots,\underbrace{\iota(\iota(\cdots\iota}_{k-1}(\nu_0)\cdots))\label{eq:sequence}
    \end{equation} 
    are all admissible, after which we can invoke condition 1) of Lemma~\ref{lem:finalgrid} to conclude that {\sc FilteredPCA} outputs a function $\td{F}$ for which $\norm{F - \td{F}} \le O(\epsilon)$. It is clear that for $\nu$ sufficiently small, $\iota$ is increasing in $\nu$. So it suffices to choose $\nu_0$ sufficiently small that the last number in the sequence \eqref{eq:sequence} is admissible. 

    Then the last number in \eqref{eq:sequence} is at most \begin{equation}
        \Max{\left(\beta^{\sum^{k-1}_{j=0}(1 - 1/k)^j}\cdot \nu_0^{(1 - 1/k)^k}\right)}{(k^k\nu_0)} \le \Max{\left(\beta^k \cdot \nu_0^{1/e}\right)}{(k^k\nu_0)}.
    \end{equation}
    If $F$ satisfies Assumption~\ref{assume:bounds} only and we take $C_* = \Cjunta$, then \begin{equation}
        \beta^k = (c'/\ulam^2)^k \cdot O\left(k^{2k}\cdot(kM^4/c^2)^{k-1}\right),
    \end{equation}
    so for \begin{equation}
        \nu_0 \triangleq \poly(k^k,1/\ulam^k,M^k,\Lip/\epsilon)^{-1} = \poly(e^{k^3\Lip^2/\epsilon^2},M^k)^{-1}
    \end{equation}
    sufficiently small, we have that $\Max{(\beta^k\cdot \nu_0^{1/e})}{(k^k\nu_0)}$ is admissible.

    And because in each of the at most $k$ iterations of the main loop of {\sc FilteredPCA}, \begin{equation}
        N = O(\brc{\Max{d}{\log(Mk/\delta)}}/\xi(\nu_0)^2) \le d\log(1/\delta) \poly(e^{k^3\Lip^2/\epsilon^2},M^k)
    \end{equation}
    samples are drawn, the final sample complexity is $d\log(1/\delta) \poly(e^{k^3\Lip^2/\epsilon^2},M^k)$ as claimed. The runtime is dominated by the at most $M^{M^2}O(1/\sqrt{\nu_0})^{\ell} = M^{M^2}\cdot \poly(e^{k^4\Lip^2/\epsilon^2},M^{k^2})$ calls to {\sc ApproxBlockSVD}, one for each element of $\calL$ output by {\sc EnumerateKickers}, (note that the runtime and sample complexity cost of running {\sc EnumerateKickers} at the very end is of much lower order). As there is a matrix-vector oracle for the matrices on which we run {\sc ApproxBlockSVD} which takes time $O(d^2)$, by Fact~\ref{thm:power-method} each of these calls takes, up to lower order factors that will be absorbed elsewhere, $\td{O}(d^2\log(1/\delta))$ time, so we conclude that {\sc FilteredPCA} runs in time \begin{equation}
        \td{O}(d^2\log(1/\delta))\cdot M^{M^2}\cdot \poly(e^{k^4\Lip^2/\epsilon^2},M^{k^2})
    \end{equation}
    as claimed.

    If $F$ satisfies Assumptions~\ref{assume:bounds} and \ref{assume:bounds2} and we take $C_* = \Cnet$, then \begin{equation}
        \beta^k = (c'/\ulam^2)^k \cdot O\left(k^{2k}\left(\frac{2^{O(L)}\normbound^{2L+4}k2^{4S}}{c^2\Lip^2}\right)^{k-1}\right),
    \end{equation}
    where we have used that $M\le 2^S$ for size-$S$ ReLU networks. So for \begin{equation}
    \nu_0 \triangleq \poly(k^k, 1/\ulam^k, 2^{kS}, (\normbound^{L+2}/\Lip)^k, \Lip/\epsilon)^{-1} = \poly(e^{k^3\Lip^2/\epsilon^2}, 2^{kS}, (\normbound^{L+2}/\Lip)^k)
    \end{equation}
    sufficiently small, we have that $\Max{(\beta^k\cdot \nu_0^{1/e})}{(k^k\nu_0)}$ is admissible.

    And because in each of the at most $k$ iteration of the main loop of {\sc FilteredPCA}, \begin{equation}
        N = O(\brc{\Max{d}{\log(2^S k/\delta)}}/\xi(\nu_0)^2) \le d\log(1/\delta) \poly(e^{k^3\Lip^2/\epsilon^2},2^{kS},\normbound^{(L+2)k}/\Lip^k)
    \end{equation}
    samples are drawn, the final sample complexity is $d\log(1/\delta) \poly(e^{k^3\Lip^2/\epsilon^2},2^{kS},\normbound^{(L+2)k}/\Lip^k)$ as claimed. The runtime is dominated by the at most $O(1/\sqrt{\nu_0})^{O(S^2)} = \poly(e^{k^3S^2\Lip^2/\epsilon^2},2^{k S^3},\normbound^{(L+2)kS^2}/\Lip^{k S^2})$ calls to {\sc ApproxBlockSVD}, one for each element of $\calL$ output by {\sc EnumerateNetworks} (note that the runtime and sample complexity cost of running {\sc EnumerateNetworks} at the very end is of much lower order). Each of these calls takes, up to lower order factors that will be absorbed elsewhere, $\td{O}(d^2\log(1/\delta))$ time, so we conclude that {\sc FilteredPCA} runs in time \begin{equation}
    \td{O}(d^2\log(1/\delta))\cdot \poly(e^{k^3S^2\Lip^2/\epsilon^2},2^{k S^3},(\normbound^{L+2}/\Lip)^{kS^2})
    \end{equation}
    as claimed.
\end{proof}


\begin{remark}[Comparison to \cite{chen2020learning}]\label{remark:compare}
	Here we briefly discuss what goes wrong if one simply tries mimicking the approach of \cite{chen2020learning}. Provided one has already recovered some (orthonormal) directions $w_1,...,w_{\ell}$ spanning a subspace $W\subset V$, one would consider the matrix \begin{equation}
	\vM^W_{\mathsf{CM}} \triangleq \Pi_{W^{\perp}}\E[x,y]*{\bone{\abs{y} > \tau \wedge \norm{\Pi_W x}^2 \le \alpha}\cdot (xx^{\top} - \Id)}\Pi_{W^{\perp}}
\end{equation} for some $\alpha,\tau>0$. The motivation for conditioning on $\norm{\Pi_W x}^2 \le \alpha$ is that we now have \begin{equation}
	\iprod{\Pi_{V\backslash W}, \vM^W_{\mathsf{CM}}} = \E[x,y]*{\bone{\abs{y} > \tau \wedge \norm{\Pi_W x}^2 \le \alpha}\cdot(\norm{\Pi_{V\backslash W} x}^2 - (k - \ell))},
\end{equation} and if one could choose $\tau$ strictly greater than the supremum of $\abs{F(x)}$ over all $x$ for which $\norm{\Pi_{W}x}^2 \le \alpha$ and $\norm{\Pi_{V\backslash W}x}^2 \le 2(k-\ell)$, then we would conclude that \begin{equation}
	\iprod{\Pi_{V\backslash W}, \vM^W_{\mathsf{CM}}} \ge (k-\ell)\cdot \Pr*{\abs{y} > \tau \wedge \norm{\Pi_W x} \le \alpha} \label{eq:cmfilter}
\end{equation} and it would suffice to lower bound the probability on the right-hand side of \eqref{eq:cmfilter}. This is precisely the route taken by \cite{chen2020learning} for learning low-degree polynomials, but in the case of ReLU networks, it is not hard to devise functions $F$ for which the probability on the right-hand side of \eqref{eq:cmfilter} is zero for such choices of $\tau$, e.g. if $d = k = 2$, $\ell = 1$, $v_1 = e_1$, and \begin{equation}
	F(x) \triangleq \phi(x/\alpha + y) - \phi(-x/\alpha + y).
\end{equation}
\end{remark}

\bibliographystyle{alpha}
\bibliography{biblio}

\appendix


\section{Deferred Proofs}

\subsection{Concentration for Piecewise Linear Functions}
\label{app:estimate_var}

\begin{lemma}\label{lem:estimate_var}
For any $\delta>0$ and any $t \le \Lip^2 k$, the following holds. Let $F:\R^d\to\R$ be a $\Lip$-Lipschitz kicker with relevant subspace $V$ of dimension $k$. Then for samples $x_1,...,x_N\sim\calN(0,\Id)$, where $N = \Theta\left({(\mu + \Lip^2 k)^2\log(1/\delta)}/t^2\right)$, the empirical estimate $\wh{\sigma}^2 \triangleq \frac{1}{N}\sum_i F(x_i)^2$ satisfies \begin{equation}
	\abs*{\E[x\sim\calN(0,\Id)]{F(x)^2} - \wh{\sigma}^2} \le t
\end{equation} with probability at least $1 - \delta$.
\end{lemma}

\begin{proof}
	As $F$ is $\Lip$-Lipschitz and continuous piecewise-linear, by Theorem~\ref{thm:tropical} and Lemma~\ref{lem:normbound} it has a lattice polynomial representation $\max_{j\in[m]}\min_{i\in\calI_j}\iprod{u_i,\cdot}$ for some clauses $\brc{\calI_j}$ and vectors $\brc{u_i}$ for which $\norm{u_i} \le \Lip$. In particular, by Cauchy-Schwarz, $\abs{F(x)} \le \Lip\norm{x}$ for all $x$. 
	Now define the function $G(x)\triangleq F(x)^2 - \mu$ where $\mu\triangleq \E[x\sim\calN(0,\Id)]{F(x)^2}$. We can therefore naively upper bound the moments of $G$ by \begin{equation}
		\E{|G|^t}^{1/t} \le \mu + \E{F^{2t}}^{1/t} \le \mu + \Lip^2\cdot \E[x\sim\calN(0,\Pi_V)]{\norm{x}^{2t}}^{1/t} \le \mu + O(\Lip^2 k)\cdot (t-1)\label{eq:hyper}
	\end{equation} for all $t\ge 2$, where the last step follows by standard hypercontractivity.
	Furthermore, $\E{\abs{G}}\le 2\mu$. For $x\sim\calN(0,\Id)$, $G(x)$ is therefore a sub-exponential, mean-zero random variable with sub-exponential norm $K \triangleq O(\mu+\Lip^2 k)$, so by Fact~\ref{fact:bernstein} and the bound on $t$ in the hypothesis, for $N = \Theta(K^2\log(1/\delta)/t^2)$, the claim follows.
\end{proof}

\subsection{Representing Boolean Functions as ReLU Networks}
\label{app:represent}

\begin{lemma}\label{lem:representation}
For any function $F:\brc{\pm 1}^n\to\brc{\pm 1}$, there exists a set of weight matrices $\vW_0,...,\vW_{n-1}$ for which $F(x) = \vW_{n-1}\phi(\vW_{n-2}\phi(\cdots\phi(\vW_0x)\cdots))$ for all $x\in\brc{\pm 1}^n$.
\end{lemma}

\begin{proof}
	From the Fourier expansion of $F$ as $F(x) = \sum_S \wh{F}[S]\prod_{i\in S}x_i$, we see that it suffices to show how to represent any Fourier basis function $\prod_{i\in S}x_i$ with a ReLU network with depth $n$. We first show how to represent the function $x_1x_2$. Observe that for any $x_1,x_2\in\brc{\pm 1}$, we have that \begin{equation}
		x_1\cdot x_2 = \phi(x_1 + x_2) + \phi(-x_1-x_2) - \phi(x_2) - \phi(-x_2),\label{eq:multiply}
	\end{equation}
	which is a two-layer neural network. Suppose inductively that for some $1\le m< n$, there exist weight matrices $\vW'_0,\ldots,\vW'_{m-1}$ for which $\prod^m_{i=1}x_i = \vW'_{m-1}\phi(\vW'_{m-2}\phi(\cdots\phi(\vW'_0x)\cdots))$ for all $x\in\brc{\pm 1}^n$.
	Then to compute $\prod^{m+1}_{i=1}x_i$, we can use \eqref{eq:multiply} to conclude that \begin{equation}
		\prod^{m+1}_{i=1}x_i = \phi\left(\prod^m_{i=1}x_i + x_{m+1}\right) + \phi\left(-\prod^m_{i=1}x_i - x_{m+1}\right) - \phi(x_{m+1}) - \phi(-x_{m+1}).
	\end{equation} It is clear that this can be represented as a ReLU network with depth $m+1$.
\end{proof}

\end{document}